%% file: main.tex
\documentclass[jmlr]{article}
\usepackage{jmlr}
\usepackage[utf8]{inputenc} %
\usepackage[T1]{fontenc}    %
\usepackage[scaled=0.92]{helvet}
\usepackage{hyperref}       %
\usepackage{url}            %
\usepackage{booktabs}       %
\usepackage{amsfonts}       %
\usepackage{nicefrac}       %
\usepackage{microtype}      %
\usepackage{xcolor}         %

\usepackage{amsthm}
\usepackage{amssymb}
\usepackage{amsmath}
\usepackage{xspace}
\usepackage{ifthen}
\usepackage[skip=2pt]{caption} 
\usepackage{balance}
\usepackage{bm}

\usepackage{floatrow}
\newfloatcommand{capbtabbox}{table}[][\FBwidth]
\usepackage{wrapfig}
\usepackage{enumitem}

\input{macros.tex}

\newtheorem{example}{Example}
\newtheorem{axiom}{Axiom}

\providecommand{\Norm}[2][]{\ensuremath{%
\ifthenelse{\equal{#1}{}}{\|{#2}\|}{\|{#2}\|_{{#1}}}}\xspace}
\providecommand{\SetCard}[1]{\ensuremath{| #1 |}\xspace}
\providecommand{\SET}[1]{\ensuremath{\{ #1 \}}\xspace}

\providecommand{\Set}[2]{\ensuremath{\SET{#1 \mid #2}}\xspace}

\providecommand{\Kth}[1]{\ensuremath{{#1}^{\rm th}}}

\providecommand{\PROB}{\ensuremath{\mathbb{P}}\xspace}
\providecommand{\Prob}[2][]{\ensuremath{%
\ifthenelse{\equal{#1}{}}{\PROB[#2]}{\PROB_{#1}\left[#2\right]}}\xspace}
\providecommand{\ProbC}[3][]{\Prob[#1]{#2\;|\;#3}}
\providecommand{\Expect}[2][]{\ensuremath{%
\ifthenelse{\equal{#1}{}}{\mathbb{E}}{\mathbb{E}_{#1}}%
\left[#2\right]}\xspace}

\providecommand{\VAR}{\ensuremath{{\rm Var}}\xspace}
\providecommand{\Var}[1]{\ensuremath{\VAR[#1]}\xspace}

\providecommand{\Event}[2][]{\ensuremath{\ifthenelse{\equal{#1}{}}{%
\mathcal{#2}}{\mathcal{#2}_{{#1}}}}\xspace}

\newcommand{\LPlabel}{}

\newenvironment{LP}[3][]{%
\renewcommand{\LPlabel}{#1}
\ifthenelse{\equal{\LPlabel}{}}{%
\[ \begin{array}{ll}
\mbox{#2} & \;\;#3 \\
\mbox{subject to} & \begin{array}[t]{ll}%
}{%
\begin{equation} \begin{array}{ll}
\mbox{#2} & \;\;#3 \\
\mbox{subject to} & \begin{array}[t]{ll}%
}}{%
\ifthenelse{\equal{\LPlabel}{}}{%
\end{array} \end{array} \]}{%
\end{array} \end{array} \label{\LPlabel} \end{equation}}%
}

\newcommand{\FPAR}[1][]{\ensuremath{%
\ifthenelse{\equal{#1}{}}{\phi}{%
\ifthenelse{\equal{#1}{'}}{\phi'}{%
\phi_{#1}}}}\xspace}

\newcommand{\Des}[1]{\ensuremath{v_{#1}}\xspace}
\newcommand{\DesV}{\ensuremath{\bm{v}}\xspace}
\newcommand{\JDIST}{\ensuremath{\Gamma}\xspace}
\newcommand{\JDISTP}{\ensuremath{\JDIST'}\xspace}
\newcommand{\JDistA}[1]{\ensuremath{\JDIST(#1)}\xspace}

\newcommand{\IsBest}[3][]{\ensuremath{%
\ifthenelse{\equal{#1}{}}{\mathcal{M}_{#2,#3}}{\mathcal{M}^{(#1)}_{#2,#3}}}\xspace}

\newcommand{\FD}{OPT/TS-Mixing\xspace}

\newcommand{\agentSet}{\ensuremath{\mathcal{X}}\xspace}

\newcommand{\detrank}{\ensuremath{\sigma}\xspace}
\newcommand{\detrankA}[1]{\ensuremath{\detrank(#1)}\xspace}
\newcommand{\detrankOPT}[1][]{\ensuremath{%
\ifthenelse{\equal{#1}{}}{\detrank^{*}}{\detrank^{*}_{#1}}}\xspace}
\newcommand{\detrankClass}{\ensuremath{\Sigma(\agentSet)}\xspace}

\newcommand{\probrank}[1][]{\ensuremath{%
\ifthenelse{\equal{#1}{}}{\pi}{\pi_{#1}}}\xspace}
\newcommand{\probrankA}[2][]{\ensuremath{\probrank[#1](#2)}\xspace}
\newcommand{\probrankTS}[1][]{\ensuremath{%
\ifthenelse{\equal{#1}{}}{\probrank^{\mathrm{TS}}}{\probrank^{\mathrm{TS}}_{#1}}}\xspace}
\newcommand{\probrankOPT}[1][]{\ensuremath{%
\ifthenelse{\equal{#1}{}}{\probrank^{*}}{\probrank^{*}_{#1}}}\xspace}
\newcommand{\probrankFD}[1]{\ensuremath{\probrank^{\mathrm{Mix}, #1}}\xspace}
\newcommand{\probrankLP}[1]{{\ensuremath{\probrank^{\mathrm{LP}, #1}}\xspace}}
\newcommand{\probrankClass}{\ensuremath{{\Pi(\agentSet)}}\xspace}

\newcommand{\InPos}[3][]{\ensuremath{%
\ifthenelse{\equal{#1}{}}{p_{#2,#3}}{p^{(#1)}_{#2,#3}}}\xspace}
\newcommand{\InPosMat}[1][]{\ensuremath{%
\ifthenelse{\equal{#1}{}}{\mathcal{P}}{\mathcal{P}^{(#1)}}}\xspace}

\newcommand{\posweight}[1]{\ensuremath{w_{#1}}\xspace}

\newcommand{\Util}{U\xspace}
\newcommand{\UtilA}[2]{\ensuremath{\Util(#1 \, | \, #2)}\xspace}

\newcommand{\Dataset}{\ensuremath{\mathcal{D}}\xspace}
\newcommand{\multparamsVec}[1][]{\ensuremath{%
\ifthenelse{\equal{#1}{}}{\bm{\theta}}{\bm{\theta}_{#1}}}\xspace}
\newcommand{\multparams}[2][]{\ensuremath{
\ifthenelse{\equal{#2}{}}{\theta_{#1}}{\theta_{#1,#2}}}\xspace}
\newcommand{\ExpRelevance}[1]{\ensuremath{\overline{\Des{}}_{#1}}\xspace}
\newcommand{\scaleparam}{\ensuremath{s}\xspace}
\newcommand{\countratings}[1]{\ensuremath{\bm{N}_{#1}}\xspace}
\newcommand{\countone}[2]{\ensuremath{N_{#1,#2}}\xspace}
\newcommand{\Dir}[1]{\ensuremath{\text{Dir}(#1)}\xspace}
\newcommand{\Bernoulli}[1]{\ensuremath{\text{Bernoulli}(#1)}\xspace}
\newcommand{\dirichletparams}{\ensuremath{\bm{\alpha}}\xspace}
\newcommand{\dpar}[1]{\ensuremath{\alpha_{#1}}\xspace}

\newcommand{\numMCSamples}{\ensuremath{5 \cdot 10^4}\xspace}
\newcommand{\scaleparamValue}{\ensuremath{1}\xspace}
\newcommand{\numMoviesPerGenre}{\ensuremath{40}\xspace}
\newcommand{\percentSampleDataset}{\ensuremath{10\%}\xspace}
\newcommand{\selectedgenre}{``Comedy''\xspace}
\newcommand{\numruns}{20\xspace}

\newcommand{\Similarity}[2]{\ensuremath{
\ifthenelse{\equal{#2}{}}{\mu_{#1}}{\mu_{#1,#2}}}\xspace}
\newcommand{\SimilaritySample}[2]{\ensuremath{
\ifthenelse{\equal{#2}{}}{\hat{\mu}_{#1}}{\hat{\mu}_{#1,#2}}}\xspace}
\newcommand{\sd}[1]{\ensuremath{\delta_{#1}}\xspace}
\newcommand{\controlusers}{\ensuremath{\mathcal{U}_{\probrankOPT}}}
\newcommand{\treatmentusers}{\ensuremath{\mathcal{U}_{\probrankTS}}}

\newcommand{\probrankMix}{\ensuremath{\probrank[\mathrm{Mix}]}\xspace}

\title{Fairness in Ranking under Uncertainty}

\author{%
  \name Ashudeep Singh   \email ashudeep@cs.cornell.edu\\
  \addr Cornell University, Ithaca, NY\\
  \AND
  \name David Kempe \email david.m.kempe@gmail.com\\
  \addr University of Southern California, Los Angeles, CA\\
  \AND
  \name Thorsten Joachims \email tj@cs.cornell.edu \\
  \addr Cornell University, Ithaca, NY \\
}
\begin{document}

\maketitle

\begin{abstract}
  \input{abstract}
\end{abstract}

\section{Introduction} \label{sec:introduction}
\input{introduction}
\section{Related Work} \label{sec:related-work}
\input{related}
\section{Ranking with Uncertain Merits} \label{sec:problem}
\input{problem}

\section{Optimal and Fair Policies} \label{sec:policies}
\input{policies}
\section{Experimental Evaluation: MovieLens Dataset} \label{sec:experiments}
\input{experiments}

\section{Real-World Experiment: Paper Recommendation} \label{sec:real-world}
\input{real_world_experiments}

\section{Conclusions and Future Work} \label{sec:conclusions}
\input{conclusions}

\section*{Acknowledgements}

This research was supported in part by NSF Award IIS-2008139 and IIS-1901168. We thank Aleksandra Korolova, Cris Moore, Stephanie Wykstra and anonymous reviewers for useful discussions and feedback.
Any opinions, findings, and conclusions or recommendations expressed in this material are those of the author(s) and do not necessarily reflect the views of the National Science Foundation.

\bibliography{main} %

\end{document}

%% file: macros.tex
\newcommand{\MovieSet}{\mathcal{S}}
\newcommand{\Rating}{\textrm{Rating}}

\newcommand{\Q}{\mathcal{Q}}

 \usepackage[bb=boondox]{mathalfa}
 \usepackage{subcaption}
\usepackage[export]{adjustbox}
\usepackage{empheq}
\usepackage{enumitem}
\usepackage{arydshln}

\usepackage{mathtools}
\usepackage{appendix}

\makeatletter
\renewcommand*{\p@section}{\S\,}
\renewcommand*{\p@subsection}{\S\,}
\renewcommand*{\p@subsubsection}{\S\,}
\makeatother


\DeclareMathOperator{\argsort}{argsort}

\newcommand{\ind}{\mathbb{1}}

\allowdisplaybreaks %

\renewcommand{\cite}{\citep}
\usepackage{xcolor}

\newcommand{\R}{\mathcal{R}}

\usepackage{amsthm}
\theoremstyle{definition}
\newtheorem{definition}{Definition}[section]
\newtheorem{corollary}{Corollary}[section]

\newtheorem{proposition}{Proposition}[section]
\newtheorem{lemma}{Lemma}[section]

%% file: abstract.tex
Fairness has emerged as an important consideration in algorithmic decision making. Unfairness occurs when an agent with higher merit obtains a worse outcome than an agent with lower merit. Our central point is that a primary cause of unfairness is uncertainty. A principal or algorithm making decisions never has access to the agents' true merit, and instead uses proxy features that only imperfectly predict merit (e.g., GPA, star ratings, recommendation letters). None of these ever fully capture an agent's merit; yet existing approaches have mostly been defining fairness notions directly based on observed features and outcomes.

Our primary point is that it is more principled to acknowledge and model the uncertainty explicitly. The role of observed features is to give rise to a posterior distribution of the agents' merits. We use this viewpoint to define a notion of approximate fairness in ranking. We call an algorithm \FPAR-fair (for $\FPAR \in [0,1]$) if it has the following property for all agents $x$ and all $k$: if agent $x$ is among the top $k$ agents with respect to \emph{merit} with probability at least $\rho$ (according to the posterior merit distribution), then the algorithm places the agent among the top $k$ agents in its \emph{ranking} with probability at least $\FPAR \rho$.

We show how to compute rankings that optimally trade off approximate fairness against utility to the principal. In addition to the theoretical characterization, we present an empirical analysis of the potential impact of the approach in simulation studies. For real-world validation, we applied the approach in the context of a paper recommendation system that we built and fielded at the KDD 2020 conference.

%% file: introduction.tex
Fairness is an important consideration in decision-making, in particular when a limited resource must be allocated among multiple agents by a principal (or decision maker).
A widely accepted tenet of fairness is that if an agent B does not have stronger merits for the resource than A, then B should not get more of the resource than A. Depending on the context,
\emph{merit} could be a qualification (e.g., job performance), a need (e.g., disaster relief), or some other measure of eligibility.

The motivation for our work is that uncertainty about merits is a primary reason that a principal's allocations can violate this tenet and thereby lead to unfair outcomes.
Were agents' merits fully observable, it would be both fair and in the principal's best interest to rank agents by their merit.
However, actual merits are practically always unobservable.
Consider the following standard algorithmic decision making environments:
(1) An e-commerce or recommender platform (the principal) displays items (the agents) in response to a user query. An item's merit is the utility the user would derive from it, whereas the platform can only observe numerical ratings, text reviews, the user's past history, and similar features.
(2) A job recommendation site or employer (the principal) wants to recommend/hire one or more applicants (the agents). The merit of an applicant is her (future) performance on the job over a period of time, whereas the site or employer can only observe (past) grades, test scores, recommendation letters, performance in interviews, and similar assessments.

In both of these examples --- and essentially all others in which algorithms are called upon to make allocation decisions between agents --- uncertainty about merit is unavoidable, and arises from multiple sources:
(1) the training data of a machine learning algorithm is a random sample,
(2) the features themselves often come from a random process, and
(3) the merit itself may only be revealed in the future after a random
process (e.g., whether an item is sold or an employee performs well).
Given that decisions will be made in the presence of uncertainty, it is important to define the notion of \emph{fairness} under uncertainty.
Extending the aforementioned tenet that ``if agent B has less merit than A, then B should not be treated better than A,'' we state the following generalization to uncertainty about merit, first for just two agents:

\begin{axiom} \label{ax:two-agent-fairness}
  If A has merit greater than or equal to B with probability at least $\rho$, then a fair policy should treat A at least as well as B with probability at least $\rho$.
\end{axiom}

This being an axiom, we cannot offer a mathematical justification. It captures an inherent sense of fairness in the absence of enough information, and it converges to the conventional tenet as uncertainty is reduced.
In particular, consider the following situation: two items A, B with 10 reviews each have average star ratings of 3.9 and 3.8, respectively; or two job applicants A, B have GPAs of 3.9 and 3.8. While this constitutes some (weak) evidence that A may have more merit than B, this evidence leaves substantial uncertainty. The posterior merit distributions based on the available information should reflect this uncertainty by having non-trivial variance; our axiom then implies that A and B must be treated similarly to achieve fairness.
In particular, it would be highly unfair to \emph{deterministically} rank A ahead of B (or vice versa).
Our main point is that this uncertainty, rather than the specific numerical values of 3.9 and 3.8, is the reason why a mechanism should treat A and B similarly.

\subsection{Our Contributions}
We study fairness in the presence of uncertainty specifically for the generalization where the principal must rank $n$ items.
Our main contribution is the fairness framework, giving definitions of fairness in ranking in the presence of uncertainty.
This framework, including extensions to approximate notions of fairness, is presented and discussed in depth in Section~\ref{sec:problem}.
We believe that uncertainty of merit is one of the most important sources of unfairness, and modeling it explicitly and axiomatically is key to addressing it.

Next, in Section~\ref{sec:policies}, we present algorithms for a principal to achieve (approximately) fair ranking distributions. A simple algorithm the principal may use to achieve approximate fairness is to mix between an optimal (unfair) ranking and (perfectly fair) Thompson sampling. We show that this policy is not optimal for the principal's utility, and we present an efficient LP-based algorithm that achieves an optimal ranking distribution for the principal, subject to an approximate fairness constraint.

We next explore empirically to what extent a focus on fairness towards the agents reduces the principal's utility. We do so with two extensive sets of experiments: one described in Section~\ref{sec:experiments} on existing data, and one described in Section~\ref{sec:real-world} ``in the wild.'' In the first set of experiments, we consider movie recommendations based on the standard MovieLens dataset and investigate to what extent fairness towards movies would result in lower utility for users of the system.
The second experiment was carried out at the 2020 ACM SIGKDD Conference on Knowledge Discovery and Data Mining, where we implemented and fielded a paper recommendation system. Half of the conference attendees using the system received rankings that were modified to ensure greater fairness towards papers, and we report on various metrics that capture the level of engagement of conference participants based on which group they were assigned to.

The upshot of our experiments and theoretical analysis is that in the settings we have studied, high levels of fairness can be achieved at a small loss in utility for the principal and the system's users. 

%% file: related.tex
As algorithmic techniques, especially machine learning, find widespread applications in decision making, there is notable interest in understanding its societal impacts. While algorithmic decisions can counteract existing biases by preventing human error and implicit bias, data-driven algorithms may also create new avenues for introducing unintended bias \cite{barocas2016big}. There have been numerous attempts to define notions of fairness in the supervised learning setting, especially for binary classification and risk assessment \cite{calders2009building,zliobaite2015relation,dwork2012fairness,hardt2016equality,mehrabi2019survey}. The group fairness perspective imposes constraints like demographic parity \cite{calders2009building, zliobaite2015relation} and equalized odds \cite{hardt2016equality}.
Follow-up work has proposed techniques for implementing fairness through pre-processing methods \cite{calmon2017optimized,lum2016statistical}, in process while learning the model \cite{zemel2013learning, woodworth2017learning, zafar2017fairness} and post-processing methods \cite{hardt2016equality,pleiss2017fairness,kim2019multiaccuracy}, in addition to causal approaches to fairness \cite{kilbertus2017avoiding, kusner2017counterfactual}.

Individual fairness, on the other hand, is concerned with comparing the outcomes of agents directly, not in aggregate.
Specifically, the individual fairness axiom states that two individuals similar with respect to a task should receive similar outcomes \cite{dwork2012fairness}. While the property of individual fairness is highly desirable, it is hard to define precisely; in particular, it is highly dependent on the definition of a suitable similarity notion. Although similar in spirit, our work sidesteps this need to define a similarity metric between agents in the feature space. Rather, we view an agent's features solely as noisy signals about the agent's merit and posit that a comparison of these merits --- and the principal's uncertainty about them --- should determine the relative ranking. Individual fairness definitions have also been adopted in online learning settings such as stochastic multi-armed bandits \cite{patil2020achieving,heidari2018preventing, schumann2019group, celis2017ranking}, where the desired property is that a worse arm is never ``favored'' over a better arm despite the algorithm's uncertainty over the true payoffs \cite{Joseph2016}, or a smooth fairness assumption that a pair of arms be selected with similar probability if they have a similar payoff distribution \cite{liu2017calibrated}. 
While these definitions are derived from the same tenet of fairness as Axiom~\ref{ax:two-agent-fairness} for a pair of agents, we extend it to rankings, where $n$ agents are compared at a time.

Rankings are a primary interface through which machine learning models support human decision making, ranging from recommendation and search in online systems to machine-learned assessments for college admissions and recruiting.
One added difficulty with considering fairness in the context of rankings is that the decision for an agent (where to rank that agent) depends not only on their own merits, but on others' merits as well \cite{dwork2019learning}. The existing work can be roughly categorized into three groups: Composition-based, opportunity-based, and evidence-based notions of fairness. The notions of fairness based on the composition of the ranking operate along the lines of demographic parity \cite{zliobaite2015relation, calders2009building}, proposing definitions and methods that minimize the difference in the (weighted) representation between groups in a prefix of the ranking \cite{yang2016measuring, celis2017ranking, asudehy2017designing, zehlike2017fa, mehrotra2018towards,zehlike2020reducing}. Other works argue against the winner-take-all allocation of economic opportunity (e.g., exposure, clickthrough, etc.) to the ranked agents or groups of agents, and that the allocation should be based on a notion of merit \cite{singh2018fairness, biega2018equity, diaz2020evaluating}. Meanwhile, the metric-based notions equate a ranking with a set of pairwise comparisons, and define fairness notions based on parity of pairwise metrics within and across groups \cite{kallus2019fairness, beutel2019fairness, narasimhan2020pairwise, lahoti2019operationalizing}. Similar to pairwise accuracy definitions, evidence-based notions such as \cite{dwork2019learning} propose semantic notions such as \emph{domination-compatibility} and \emph{evidence-consistency}, based on relative ordering of subsets within the training data. Our fairness axiom combines the opportunity-based and evidence-based notions by stating that the economic opportunity allocated to the agents must be consistent with the existing evidence about their relative ordering. 

Ranking has been widely studied in the field of Information Retrieval (IR), mostly in the context of optimizing user utility. The \emph{Probability Ranking Principle (PRP)} \cite{robertson1977probability}, a guiding principle for ranking in IR, states that user utility is optimal when documents (i.e., the agents) are ranked by expected values of their estimated relevance (merit) to the user. While this certainly holds when the estimates are unbiased and devoid of uncertainty, we argue that it leads to unfair rankings for agents about whose merits the model might be uncertain. While the research on diversified rankings in IR appears related, in comparison to our work, the goal there is to maximize user utility alone by handling uncertainty about the user's information needs \cite{radlinski2009redundancy} and to avoid redundancy in the ranking \cite{Clarke:2008:NDI:1390334.1390446,Carbonell:1998:UMD:290941.291025}. Besides ranking diversity, IR methods have dealt with uncertainty in relevance that comes via users' implicit or explicit feedback \cite{penha2021calibration,soufiani2012random}, as well as stochasticity arising from optimizing over probabilistic rankings instead of discrete combinatorial structures \cite{taylor2008softrank,burges2005learning}. It is only recently that there has been an interest in developing evaluation metrics \cite{diaz2020evaluating} and learning algorithms \cite{singh2019policy,morik2020controlling} that use stochastic ranking models to deal with unfair exposure. 

Additional recent strands of work on fairness in selection problems focus on fairly selecting individuals distributed across different groups in the presence of group-based implicit bias \cite{kleinberg2018selection,celis2020interventions}, noisy sensitive attributes \cite{mehrotra2021mitigating}, or incomparable merits across different groups \cite{Kearns2017}. \citet{Kearns2017} present a way to fairly select $k$ individuals distributed across $d$ populations where each population can be sorted by merit without uncertainty but merit in one population cannot be directly compared to merit in another. Hence, they propose using the true CDF rank as a derived merit criterion that can be compared.
There has also been recent interest in studying the effect of uncertainty regarding sensitive attributes, labels and other features used by the machine learning model on the accuracy-based fairness properties of the model \cite{ghosh2021fair, prost2021measuring}. In contrast, our work takes a more fundamental approach to defining a merit-based notion of fairness arising due to the presence of uncertainty when estimating merits based on fully observed features and outcomes.

%% file: problem.tex
We are interested in ranking policies for a principal (the ranking system, such as an e-commerce platform or a job portal in our earlier examples) whose goal is to rank a set \agentSet of $n$ agents (such as products or applicants). The principal observes some evidence for the merit of the agents, and must produce a distribution over rankings trading off fairness to the agents against the principal's utility.
For the agents, a higher rank is always more desirable than a lower rank.

\subsection{Rankings and Ranking Distributions}
We use $\detrankClass$ to denote the set of all $n!$ rankings, and $\probrankClass$ for the set of all distributions over $\detrankClass$.
We express a ranking $\detrank \in \detrankClass$ in terms of the agents assigned to given positions, i.e., $\detrankA{k}$ is the agent in position $k$.
A ranking distribution $\probrank\in\probrankClass$ can be represented by the $n!$ probabilities \probrankA{\detrank} of the rankings $\detrank \in \detrankClass$. However, all the information relevant for our purposes can be represented more compactly using the \emph{Marginal Rank Distribution}: we write 
$\InPos[\probrank]{x}{k} = \sum_{\detrank: \detrankA{k}=x} \probrankA{\detrank}$ for the probability under \probrank that agent $x \in \agentSet$ is in position $k$ in the ranking.
We let $\InPosMat[\probrank] = (\InPos[\probrank]{x}{k})_{x,k}$ denote the $n \times n$ matrix of all marginal rank probabilities.

The matrix \InPosMat[\probrank] is doubly stochastic, i.e., the sum of each row and column is 1. While \probrank uniquely defines \InPosMat[\probrank], the converse mapping may not be unique. However, given a doubly stochastic matrix \InPosMat, the Birkhoff-von Neumann decomposition \cite{birkhoff1946tres} can be used to compute \emph{some} ranking distribution \probrank consistent with \InPosMat, i.e., $\InPosMat[\probrank] = \InPosMat$; any consistent distribution \probrank will suffice for our purposes.

\subsection{Merit, Uncertainty, and Fairness}
The principal must determine a distribution over rankings of the agents. This distribution will be based on some evidence for the agents' merits. This evidence could take the form of star ratings and reviews of products (combined with the site visitor's partially known preferences), or GPA, test scores, and recommendation letters of an applicant. Our main departure from past work on individual fairness (following \cite{dwork2012fairness}) is that we do not view this evidence as having inherent meaning; rather, its sole role is to induce a posterior joint distribution over the agents' merits.

The merit of agent $x$ is $\Des{x} \in \mathbb{R}$, and we write $\DesV = (\Des{x})_{x\in\agentSet}$ for the vector of all agents' merits.
Based on all observed evidence, the principal can infer a distribution
\JDIST over agents' merits using any suitable Bayesian inference procedure.
Since the particular Bayesian model depends on the application, for our purposes, we merely assume that a posterior distribution \JDIST was inferred using best practices and that ideally, this model is open to verification and audit.

We write \JDistA{\DesV} for the probability of merits \DesV under \JDIST.
We emphasize that the distribution will typically not be independent over entries of \DesV\ --- for example, students' merit conditioned on observed grades will be correlated via common grade inflation if they took the same class.
To avoid awkward tie-breaking issues, we assume that $\Des{x} \neq \Des{y}$ for all distinct $x, y \in \agentSet$ and all \DesV in the support of \JDIST. This side-steps having to define the notion of top-$k$ lists with ties, and comes at little cost in expressivity, as any tie-breaking would typically be encoded in slight perturbations to the \Des{x} anyway.

We write \IsBest[\DesV]{x}{k} for the event that under \DesV, agent $x$ is among the top $k$ agents with respect to merit, i.e., that $\SetCard{\Set{x'}{\Des{x'} > \Des{x}}} < k$.
We now come to our key definition of approximate fairness.

\begin{definition}[Approximately Fair Ranking Distribution]
\label{def:approximate-fairness}
  A ranking distribution \probrank is \textit{\FPAR-fair} iff
\begin{align} \label{eqn:approximate-fairness}
  \sum_{k'=1}^k \InPos[\probrank]{x}{k'}
  & \geq \FPAR \cdot \Prob[\DesV \sim \JDIST]{\IsBest[\DesV]{x}{k}} 
\end{align}
  for all agents $x$ and positions $k$.
  That is, the ranking distribution \probrank ranks $x$ at position $k$ or above with at least a \FPAR fraction of the probability that $x$ is actually among the top $k$ agents according to \JDIST.
  Furthermore, \probrank is \emph{fair} iff it is 1-fair.
\end{definition}

The reason for defining \FPAR-approximately fair ranking distributions (rather than just fair distributions) is that fairness typically comes at a cost to the principal (such as lower expected clickthrough or lower expected performance of recommended employees).
For example, if the \Des{x} are probabilities that a user will purchase products on an e-commerce site, then deterministically ranking by decreasing $\Expect[\JDIST]{\Des{x}}$ is the principal's optimal ranking under common assumptions about user behavior; yet, being deterministic, it is highly unfair.
Our definition of approximate fairness allows, e.g., a policymaker to choose a trade-off regarding how much fairness (with resulting expected utility loss) to require from the principal.
Notice that for $\FPAR = 0$, the principal is unconstrained.

We remark that the merit values \Des{x} only matter insofar as comparison is concerned; in other words, they are used ordinally, not cardinally. This is captured by the following proposition.

\begin{proposition} \label{prop:ordinal-merit}
  Let $f: \mathbb{R} \to \mathbb{R}$ be any strictly increasing function.
  Let \JDISTP be the distribution that draws the vector $(f(\Des{x}))_{x \in \agentSet}$ with probability \JDistA{\DesV} for all \DesV; that is, it replaces each entry \Des{x} with $f(\Des{x})$.
  Then, a ranking distribution \probrank is \FPAR-fair with respect to \JDIST if and only if it is \FPAR-fair with respect to \JDISTP.
\end{proposition}

\begin{proof}
  Because $f$ is strictly increasing,
  $\Prob[\DesV \sim \JDIST]{\IsBest[\DesV]{x}{k}} = \Prob[\DesV \sim \JDISTP]{\IsBest[\DesV]{x}{k}}$ for all $x$ and $k$.
  This immediately implies the claim by examining Definition~\ref{def:approximate-fairness}.
\end{proof}

Proposition~\ref{prop:ordinal-merit} again highlights a key aspect of our fairness definition: we explicitly avoid expressing any notion of fairness when ``one agent has just `a little' more merit than the other,'' instead arguing that fairness is only truly violated when an agent with more merit is treated worse than one with less merit. In other words, fairness is inherently \emph{ordinal} in our treatment. This viewpoint has implications for a principal seeking ``high-risk high-reward'' agents, which we discuss in more depth in Section~\ref{sec:high-risk-high-reward}.

\subsection{The Principal's Utility}
\label{sec:principal-utility}
The principal's utility can be the profit of an e-commerce site or the satisfaction of its customers.
We assume that the principal's utility for a ranking \detrank with agent merits \DesV takes the form
\begin{align} 
  \UtilA{\detrank}{\DesV}
  & = \sum_{k=1}^n \posweight{k} \Des{\detrankA{k}},\label{eq:util}
\end{align}
where $\posweight{k}$ is the position weight for position $k$ in the ranking, and we assume that the $\posweight{k}$ are non-increasing, i.e., the principal derives the most utility from earlier positions of the ranking.

The assumption that the utility from each position is factorable (i.e., of the form $\posweight{k} \cdot \Des{\detrankA{k}}$) is quite standard in the literature \cite{jarvelin2002cumulated,taylor2008softrank}.
The assumption that the utility is \emph{linear} in \Des{\detrankA{k}} is in fact not restrictive at all.
To see this, assume that the principal's utility were of the form $\posweight{k} \cdot f(\Des{\detrankA{k}})$ for some strictly increasing function $f$. By Proposition~\ref{prop:ordinal-merit}, the exact same fairness guarantees are achieved when the agents' merits \Des{x} are replaced with $f(\Des{x})$; doing so preserves fairness, while in fact making the principal's utility linear in the merits.
Some common examples of utility functions falling into this general framework are DCG with  $\posweight{k} = 1/\log_2(1+k)$, Average Reciprocal Rank with $\posweight{k} = 1/k$, and Precision@K with $\posweight{k} = \ind[k\leq K]/K$. 

When the ranking and merits are drawn from distributions, the principal's utility is the expected utility under both sources of randomness:
\begin{align}
  \UtilA{\probrank}{\JDIST}
  & = \Expect[\detrank\sim\probrank, \DesV \sim \JDIST]{\UtilA{\detrank}{\DesV}}. \label{eqn:util-probrank}
\end{align}

\subsection{Discussion} \label{sec:discussion}
Our definition is superficially similar to existing definitions of individual fairness (e.g., \cite{dwork2012fairness,Joseph2016}), in that similar observable features often lead to similar outcomes.
Importantly, though, it side-steps the need to define a similarity metric between agents in the feature space.
Furthermore, it does not treat the observable attributes (such as star ratings) themselves as any notion of ``merit.''
Instead, our central point is that agents' features should be viewed \emph{solely} as noisy signals about the agents' merits and that a comparison of their merits --- and the principal's uncertainty about the merits --- should determine the agents' relative ranking.
That moves the key task of quantifying individual fairness from articulating which features should be considered relevant for similarity, to articulating what inferences can be drawn about merit from observed features.

\subsubsection{Merit as an Abstraction Boundary between Data and Fairness}

One may argue, rightfully, that from an operational perspective, our approach simply pushes the normative decisions into determining \JDIST.
For example, if the distribution \JDIST were biased in favor of or against a particular group, then the decisions of a supposedly fair algorithm (with respect to \JDIST) would in fact be unfair to that group.
However, our main point is that normative decisions \emph{should} indeed be encoded in the distribution \JDIST.
To appreciate the conceptual approach, first notice that any algorithm implicitly encodes normative decisions, merely in the outputs it produces, which will favor some agents over others.
The key question is how these normative decisions are encoded, how they can be articulated, and whether they could possibly be audited.
If they are encoded in ad hoc algorithmic choices, articulating and auditing them may be difficult.

As an example, consider an admissions officer at a university, who believes that the GPA or SAT scores of affluent applicants may be higher due to access to tutors, rather than true academic potential.
One approach to compensate for this advantage could be to subtract some (wealth-dependent) amount from an applicant's scores.
A more principled approach --- and the one we advocate --- is for the admissions officer to explicitly express the possible distribution of merits given the test scores and wealth. The suitable notion of fairness is then derived by our framework, rather than as an ad hoc choice.
A substantive discussion can then be had around the assumptions that go into the admission officer's particular choice of distribution, whether a different distribution would be more suitable, etc.

In a sense, the notion of merit, and distributions thereof, serves as a clean abstraction boundary between available data, and the desired fairness and utility.
We argue that frequently, the difficult question to address is not so much what is ``fair,'' but what the data truly reveal about an agent's merit.
The latter should be articulated by domain experts, whereas the role of computer science is to provide frameworks for deriving fair algorithms \emph{given} the merit distributions, as well as statistical approaches that may guide the derivation of \JDIST from data.

\subsubsection{Randomization, Fairness, and Single-Shot Scenarios}
In the introduction, we discussed two possible applications in which fairness is desirable: ranking of products in online e-commerce, and ranking of job applicants. We note that these two settings differ along an important dimension: e-commerce sites typically display/rank the same set of products many times over a short period of time, and the stakes each time are fairly low. On the other hand, any one particular job is a one-shot setting with high stakes.
  Intuitively, it ``feels'' like the use of randomization as a means to achieve fairness is more natural in the former setting than the latter. We discuss this issue in more depth.

First, we consider two practical reasons for feeling that randomization is more natural in repeated low-stakes settings. For one, in a high-stakes situation, a principal may be less willing to trade off utility for fairness. Moreover, if fairness is \emph{required} of the principal (rather than the principal's own goal), in a single-shot setting, it is much more difficult (or impossible) to verify that a decision was indeed made probabilistically; in contrast, for a repeated setting, statistical tools can be employed to keep the principal honest.

More fundamentally, the two settings differ in the point in time at which fairness is guaranteed. For concreteness, consider the simplest setting: two agents with identical posterior distributions vie for one position. In this case, a coin flip is \emph{ex ante} fair: before the coin flip is realized, both agents have the same probability of being selected. However, it is not fair \emph{ex post}: despite both having equal merit, one was selected, and the other was not. Contrast this with the alternative in which the same two agents compete multiple times, and a coin is flipped each time. Ex ante fairness is of course preserved, but even ex post, each agent was selected approximately the same number of times. This example may explain why randomization ``feels'' more fair for repeated than one-shot setting.

The fact that for a one-shot setting, a coin flip is only ex ante fair, however, does not obviate the need for making fair decisions in one-shot settings. There will be situations in which a principal is faced with multiple essentially indistinguishable agents and not enough positions for all of them.\footnote{For example, in college admissions, qualified applicants typically  outnumber available slots, and differences among the qualified applicants are frequently very small.} While ex ante fairness may not be completely satisfactory, it still guarantees ``more'' fairness than arbitrary deterministic tie-breaking. Indeed, one may argue that the goal of many principals is not so much to make fair decisions as to make defensible ones. For example, if applicants are ranked strictly by GPA, choosing an applicant with GPA 3.91 over one with GPA of 3.90 is essentially random tie-breaking, but with a rule that can be defended.
Our point here is that randomization should be considered as a viable alternative, if the true goal is to achieve fairness.

\subsubsection{Information Acquisition Incentives}
An additional benefit of requiring the use of fair ranking policies is that it makes the principal bear more of the cost of an inaccurate \JDIST, and thereby incentivizes the principal to improve the distribution \JDIST.
To see this at a high level, notice that if \JDIST precisely revealed merits, then the optimal and fair policies would coincide.
In the presence of uncertainty, an unrestricted principal will optimize utility, and in particular do better than a principal who is constrained to be (partially or completely) fair. Thus, a fair principal stands to gain more by obtaining perfect information. The following example shows that this difference can be substantial, i.e., the information acquisition incentives for a fair principal can be much higher.

\begin{example}
  Consider again the case of a job portal. To keep the example simple, consider a scenario in which the portal tries to recommend exactly one candidate for a position.\footnote{This can be considered a ranking problem in which the first slot has $\posweight{1} = 1$, while all other slots have weight 0.}
  There are two groups of candidates, which we call majority %
  and underrepresented minority (URM). %
  The majority group contains exactly one candidate of merit 1, all others having merit 0; the URM group contains exactly one candidate of merit $1+\epsilon$, all others having merit 0 as well.
  Due to past experience with the majority group, the portal's distribution \JDIST over merits precisely pinpoints the meritorious majority candidate, but reveals no information about the meritorious URM candidate; that is, the distribution places equal probability on each of the URM candidates having merit $1+\epsilon$.

  A utility-maximizing portal will therefore go with ``the known thing,'' obtaining utility 1 from recommending the majority candidate. The loss in utility from ignoring the URM candidates is only $\epsilon$. Now consider a portal required to be 1-fair. Because each of the URM candidates is the best candidate with probability $\nicefrac{1}{n}$ (when there are $n$ URM candidates), and the majority candidate is \emph{known} to never be the best candidate, each URM candidate \emph{must} be recommended with probability $\nicefrac{1}{n}$. Here, the uncertainty about which URM candidate is meritorious will provide the portal with a utility that is only $\nicefrac{(1+\epsilon)}{n}$.

\end{example} 

In this example, fairness strengthens the incentive for the portal to acquire more information about the URM group; specifically, to learn to perfectly identify the meritorious candidate.
Under full knowledge, the portal will now have utility $1+\epsilon$ for both the fair and the utility-maximizing policy.
For the utility-maximizing portal, this is the optimal choice; and for the fair strategy, it is perfectly fair to always select the (deterministically known) best candidate.
Thus, a portal forced to use the fair strategy stands to increase its utility by a much larger amount; at least in this example, our definition of fairness splits the cost of a high-variance distribution \JDIST more evenly between the principal and the affected agents when compared to the utility-optimizing policy, where almost all the cost of uncertainty is borne by the agents in the URM group. This drastically increases the principal's incentives for more accurate and equitable information gathering.

To what extent the insights from this example generalize to arbitrary settings (e.g., whether the principal \emph{always} stands to gain more from additional information when forced to be fairer) is a fascinating direction for future research.

\subsubsection{Ordinal Merit and High-Risk High-Reward Agents}
\label{sec:high-risk-high-reward}  
As we discussed earlier, Proposition~\ref{prop:ordinal-merit} highlights the fact that our definition of fairness only considers ordinal properties, i.e., comparisons, of merit.
This means that frequently selecting ``moonshot'' agents (those with very rare very high merit) would be considered unfair. We argue that this is not a drawback of our fairness definition; rather, if moonshot attempts are worth supporting frequently, then the definition of merit should be altered to reflect this understanding. As a result, viewing the merit definition under the prism of our fairness definition helps reveal misalignments between stated merit and actual preferences.

For a concrete example, consider two agents: agent A has known merit 1, while agent B has merit $M \gg 1$ with probability 1\% and 0 with probability 99\%. When $M > 100$, agent B has larger expected merit, but regardless of whether $M > 100$ or $M \leq 100$, a fully fair principal cannot select B with probability more than 1\%.
One may consider this a shortcoming of our model: it would prevent, for instance, a funding agency (which tries to be fair to research grant PIs) from focusing on high-risk high-reward research.
We argue that the shortcoming will typically not be in the fairness definition, but in the chosen definition of merit.
For concreteness, suppose that the status quo is to evaluate merit as the total number of citations which the funded work attracts during the next century.\footnote{This measure is chosen for simplicity of discussion, not to actually endorse this metric.}
Also, for simplicity, suppose that ``high-reward'' research is research that attracts more than 100,000 citations over the next century.
If we consider one unit of merit as 1000 citations, and assume that the typical research grant results in work attracting about that many citations, then the funding agency faces the problem from the previous paragraph, and will not be able to support PI B with probability more than 1\%. This goes against the express preference of many funding agencies for high-risk high-reward work.

However, if one truly believes that high-reward work is fundamentally different (e.g., it will change the world), then this difference should be explicitly modeled in the notion of merit.
For example, rather than ``number of citations,'' an alternative notion of merit would be ``probability that the number of citations exceeds 100,000.'' This approach would allow the agency to select PIs based on the posterior probability (based on observed attributes, such as track record and the proposal) of producing such high-impact work. Of course, in reality, different aspects of merit can be combined to define a more accurate notion of merit that reflects what society values as true merit of research.

The restrictions imposed on a principal by our framework will and should force the principal to articulate actual merit of agents carefully, rather than adding ad hoc objectives. Once merit has been clearly defined, we anticipate that the conflict between fairness and societal objectives will be significantly reduced.

\subsubsection{Other Considerations}

In extending the probabilistic fairness axiom from two to multiple agents in Equation~\eqref{eqn:approximate-fairness}, we chose to axiomatize fairness in terms of which position agents are assigned to. An equally valid generalization would have been to require for each pair $x, y$ of agents that if $x$ has more merit than $y$ with probability at least $\rho$, then $x$ must precede $y$ in the ranking with probability at least $\FPAR \cdot \rho$.
The main reason why we prefer Equation~\eqref{eqn:approximate-fairness} is computational: the only linear programs we know for the alternative approach require variables for all rankings and are thus exponential (in $n$) in size. Exploring the alternative definition is an interesting direction for future work.

%% file: policies.tex
For a distribution \JDIST over merits, let \detrankOPT[\JDIST] be the ranking which sorts the agents by expected merit, i.e., by non-increasing $\Expect[\DesV \sim \JDIST]{\Des{x}}$.
The following well-known proposition follows because the position weights $\posweight{k}$ are non-increasing; we provide a short proof for completeness.

\begin{proposition} \label{prop:sorting-optimal}
  \detrankOPT[\JDIST] is a utility-maximizing ranking policy for the principal.
\end{proposition}

\begin{proof}
  Let \probrank be a randomized policy for the principal. We will use a standard exchange argument to show that making \probrank more similar to \detrankOPT[\JDIST] can only increase the principal's utility.
  Recall that by Equation~\eqref{eqn:util-probrank}, the principal's utility under \probrank can be written as
  \begin{align*}
    \UtilA{\probrank}{\JDIST}
    &= \sum_{x\in\agentSet} \sum_k \InPos[\probrank]{x}{k}  \cdot \Expect[\DesV\sim\JDIST]{\Des{x}} \cdot \posweight{k}.
  \end{align*}
  Assume that \probrank does not sort $x$ by non-increasing $\Expect[\DesV\sim\JDIST]{\Des{x}}$. Then, there exist two positions $j < k$ and two agents $x, y$ such that $\Expect[\DesV\sim\JDIST]{\Des{x}} > \Expect[\DesV\sim\JDIST]{\Des{y}}$, and $\InPos[\probrank]{x}{k} > 0$ and $\InPos[\probrank]{y}{j} > 0$.
  Let $\epsilon = \min (\InPos[\probrank]{x}{k}, \InPos[\probrank]{y}{j}) > 0$, and consider the modified policy which subtracts $\epsilon$ from \InPos[\probrank]{x}{k} and \InPos[\probrank]{y}{j} and adds $\epsilon$ to \InPos[\probrank]{x}{j} and \InPos[\probrank]{y}{k}.
  This changes the expected utility of the policy by
  \begin{align*}
  \epsilon \cdot (
    \Expect[\DesV\sim\JDIST]{\Des{x}} \cdot \posweight{j} + 
    \Expect[\DesV\sim\JDIST]{\Des{y}} \cdot \posweight{k} - 
    \Expect[\DesV\sim\JDIST]{\Des{x}} \cdot \posweight{k} -
    \Expect[\DesV\sim\JDIST]{\Des{y}} \cdot \posweight{j} )
  \\ =  \epsilon \cdot (\posweight{j} - \posweight{k})
         \cdot (\Expect[\DesV\sim\JDIST]{\Des{x}} - 
                \Expect[\DesV\sim\JDIST]{\Des{y}}) 
    \; \geq \; 0.
  \end{align*}
  By repeating this type of update, the policy eventually becomes fully sorted, weakly increasing the utility with every step. Thus, the optimal policy must be sorted by $\Expect[\DesV\sim\JDIST]{\Des{x}}$.
\end{proof}

Assuming that the principal's expected utility can be evaluated efficiently, computing \detrankOPT[\JDIST] only requires sorting the agents by utility, and thus takes time only $O(n \log n)$.
While this policy conforms to the Probability Ranking Principle \cite{robertson1977probability}, it violates Axiom~\ref{ax:two-agent-fairness} for ranking fairness when merits are uncertain.
We define a natural solution for a 1-fair ranking distribution based on Thompson Sampling:

\begin{definition}[Thompson Sampling Ranking Distribution]
  Define \probrankTS[\JDIST] as follows: 
  first, draw a vector of merits $\DesV \sim \JDIST$,
  then rank the agents by decreasing merits in \DesV.
\end{definition}

That \probrankTS[\JDIST] is 1-fair follows directly from the definition of fairness. By definition, it ranks each agent $x$ in position $k$ with exactly the probability that $x$ has $k$-th highest merit.

\begin{proposition} \label{prop:thompson-fairness}
\probrankTS[\JDIST] is a 1-fair ranking distribution.
\end{proposition}

Furthermore, computing \probrankTS[\JDIST] only involves sampling from \JDIST and then sorting the agents by merit, so it can be efficiently performed in time $O(n \log n)$.

\subsection{Trading Off Utility and Fairness} \label{sec:fairness-day}

One straightforward way of trading off between the two objectives of fairness and principal's utility is to randomize between the two policies $\probrankTS[\JDIST]$ and $\probrankOPT[\JDIST]$.

\begin{definition}[\FD]
The \FD ranking policy $\probrankFD{\FPAR}$ randomizes between \probrankTS[\JDIST] and \probrankOPT[\JDIST] with probabilities $\FPAR$ and $1-\FPAR$, respectively.
\end{definition}

This policy inherits a runtime of $O(n \log n)$ from its two constituent policies \probrankTS[\JDIST] and \probrankOPT[\JDIST].
The following lemma gives guarantees for such randomization (but we will later see that this strategy is suboptimal).

\begin{lemma} \label{lemma:mixture}
Consider two ranking policies \probrank[1] and \probrank[2] such that \probrank[1] is \FPAR[1]-fair and \probrank[2] is \FPAR[2]-fair. A policy that randomizes between \probrank[1] and \probrank[2] with probabilities $q$ and $1-q$, respectively, is at least $(q\FPAR[1] + (1-q)\FPAR[2])$-fair and obtains expected utility $q\UtilA{\probrank[1]}{\JDIST} + (1-q)\UtilA{\probrank[2]}{\JDIST}$. 
\end{lemma}

\begin{proof}
  Both the utility and fairness proofs are straightforward. The proof of fairness decomposes the probability of agent $i$ being in position $k$ under the mixing policy into the two constituent parts, then pulls terms through the sum.
  The proof of utility uses Equation~\eqref{eqn:util-probrank} and linearity of expectations. We now give details of the proofs.

We write \probrankMix for the policy that randomizes between \probrank[1] and \probrank[2] with probabilities $q$ and $1-q$, respectively.
Using Equation~\eqref{eqn:util-probrank}, we can express the utility of \probrankMix as
\begin{align*}
\UtilA{\probrankMix}{\JDIST}
  & = \Expect[\detrank \sim \probrankMix, \DesV \sim \JDIST]{\UtilA{\detrank}{\DesV}} = \Expect[\DesV \sim \JDIST]{\sum_{\detrank} \probrankMix(\detrank) \cdot \UtilA{\detrank}{\DesV}}
  \\ & = \Expect[\DesV \sim \JDIST]{\sum_{\detrank}(q\cdot\probrank[1](\detrank) + (1-q)\cdot \probrank[2](\detrank)) \cdot \UtilA{\detrank}{\DesV}}
  \\ & = q \cdot \Expect[\DesV \sim \JDIST]{\sum_{\detrank}\probrank[1](\detrank)\cdot\UtilA{\detrank}{\DesV}} +
        (1-q)\cdot \Expect[\DesV \sim \JDIST]{\sum_{\detrank}\probrank[2](\detrank) \cdot \UtilA{\detrank}{\DesV}}
  \\ & = q \UtilA{\probrank[1]}{\JDIST} + (1-q)\UtilA{\probrank[2]}{\JDIST}.
\end{align*}

Similarly, we prove that \probrank is at least $(q\FPAR[1] + (1-q)\FPAR[2])$-fair if \probrank[1] and \probrank[2] are \FPAR[1]- and \FPAR[2]-fair, respectively: 
\begin{align*}
\sum_{k'=1}^k \InPos[\probrank]{x}{k'} & = \sum_{k'=1}^k q \cdot \InPos[{\probrank[1]}]{x}{k'} + (1-q)\cdot \InPos[{\probrank[2]}]{x}{k'}\\
  & = q \cdot \sum_{k'=1}^k \InPos[{\probrank[1]}]{x}{k'} + (1-q) \cdot \sum_{k'=1}^k \InPos[{\probrank[2]}]{x}{k'}\\
  &\geq q \FPAR[1] \cdot \Prob[\DesV \sim \JDIST]{\IsBest[\DesV]{x}{k}} + (1-q) \FPAR[2] \cdot \Prob[\DesV \sim \JDIST]{\IsBest[\DesV]{x}{k}} \\
  &= (q\cdot \FPAR[1] + (1-q)\cdot \FPAR[2]) \cdot \Prob[\DesV \sim \JDIST]{\IsBest[\DesV]{x}{k}},
\end{align*} 
where the inequality used that \probrank[1] is \FPAR[1]-fair and \probrank[2] is \FPAR[2]-fair.
Hence, we have proved that \probrank is $(q\cdot \FPAR[1] + (1-q)\cdot \FPAR[2])$-fair under \JDIST.
\end{proof}

\begin{corollary}
The ranking policy $\probrankFD{\FPAR}$ is \FPAR-fair.
\end{corollary}

By definition, $\probrankFD{\FPAR=0}$ has the highest utility among all 0-fair ranking policies. Furthermore, all 1-fair policies achieve the same utility since the fairness axiom for $\FPAR=1$ completely determines the marginal rank probabilities (Lemma~\ref{lemma:1-fair-optimal}). 

\begin{lemma}\label{lemma:1-fair-optimal}
  All 1-fair ranking policies have the same utility for the principal.
\end{lemma}

\begin{proof}
Let \probrank be a 1-fair ranking policy. By Equation~\eqref{eqn:approximate-fairness}, \probrank must satisfy the following constraints:
\begin{align}
  \sum_{k'=1}^k \InPos[\probrank]{x}{k'}
& \geq \Prob[\DesV \sim \JDIST]{\IsBest[\DesV]{x}{k}} & \mbox{for all $x$ and $k$.} \label{eqn:approximate-fairness-appendix}
\end{align}

Summing over all $x$ (for any fixed $k$), both the left-hand side and right-hand side sum to $k$; for the left-hand side, this is the expected number of agents placed in the top $k$ positions by \probrank, while for the right-hand side, it is the expected number of agents among the top $k$ in merit.
Because the weak inequality \eqref{eqn:approximate-fairness-appendix} holds for all $x$ and $k$, yet the sum over $x$ is equal, \emph{each} inequality must hold with \emph{equality}:

\begin{align*}
  \sum_{k'=1}^k \InPos[\probrank]{x}{k'}
& = \Prob[\DesV \sim \JDIST]{\IsBest[\DesV]{x}{k}} & \mbox{for all $x$ and $k$.} 
\end{align*}
This implies that

\begin{align*}
\InPos[\probrank]{x}{k} & = \Prob[\DesV \sim \JDIST]{\IsBest[\DesV]{x}{k}} - \Prob[\DesV \sim \JDIST]{\IsBest[\DesV]{x}{k-1}},
\end{align*}
which is completely determined by \JDIST.
Substituting these values of \InPos[\probrank]{x}{k} into the principal's utility, we see that it is independent of the specific 1-fair policy used.
\end{proof}

However, while $\probrankFD{\FPAR=0}$ and $\probrankFD{\FPAR=1}$ have the highest utility among 0-fair and 1-fair ranking policies, respectively, $\probrankFD{\FPAR}$ will typically not have maximum utility for the principal among all the \FPAR-fair ranking policies for other values of $\FPAR \in (0,1)$.
We illustrate this with the following example with $n=3$ agents.
\input{TS-not-optimal.tex}

\subsection{Optimizing Utility for \FPAR-Fair Rankings}
\label{sec:LP-algorithm}

We now formulate a linear program for computing the policy $\probrankLP{\FPAR}$ that maximizes the principal's utility, subject to being \FPAR-fair.
The variables of the linear program are the marginal rank probabilities \InPos[\probrank]{x}{k} of the distribution \probrank to be determined.
Then, by Equation~\eqref{eqn:util-probrank} and linearity of expectation, the principal's expected utility can be written as 
$
  \UtilA{\probrank}{\JDIST}
  = \sum_{x\in\agentSet} \sum_k \InPos[\probrank]{x}{k}  \cdot \Expect[\DesV\sim\JDIST]{\Des{x}} \cdot \posweight{k}.
$
We use this linear form of the utilities to write the optimization problem as the following LP with variables \InPos{x}{k} (omitting \probrank from the notation):
\begin{LP}[eqn:utility-maximizing-LP]{Maximize}{\!\!\!\!\sum_x \sum_k \InPos{x}{k} \cdot \Expect[\DesV \sim \JDIST]{\Des{x}} \cdot\posweight{k}}
  \!\!\!\!\sum_{k'=1}^k \InPos{x}{k'} \geq \FPAR \cdot \Prob[\DesV \sim \JDIST]{\IsBest[\DesV]{x}{k}} & \!\!\!\mbox{for all } x, k \\
  \!\!\!\!\sum_{k=1}^n \InPos{x}{k} = 1 & \!\!\!\mbox{for all } x\\
  \!\!\!\!\sum_x \InPos{x}{k} = 1 & \!\!\!\mbox{for all } k\\
  \!\!\!\!0 \leq \InPos{x}{k} \leq 1 & \!\!\!\mbox{for all } x, k.
\end{LP}

In the LP, the first set of constraints captures \FPAR-approximate fairness for all agents and positions, while the remaining constraints ensure that the marginal probabilities form a doubly stochastic matrix.

As a second step, the algorithm uses the Birkhoff-von Neumann (BvN) Decomposition of the matrix $\InPosMat = (\InPos{x}{k})_{x,k}$ to explicitly obtain a distribution \probrank over rankings such that \probrank has marginals \InPos{x}{k}.
The Birkhoff-von Neumann Theorem \cite{birkhoff1946tres} states that the set of doubly stochastic matrices is the convex hull of the permutation matrices, which means that we can write $\InPosMat = \sum_{\detrank} q_{\detrank} \InPosMat[\detrank]$,
where \InPosMat[\detrank] is the binary permutation matrix corresponding to the deterministic ranking \detrank, and the $q_{\detrank}$ form a probability distribution.
It was already shown by \citet{birkhoff1946tres} how to find a polynomially sparse decomposition in polynomial time.

Having to solve a linear program obviously makes the computation of \probrankLP{\FPAR} significantly less efficient. The computation is still efficient enough to be feasible for several hundred agents. An interesting direction for future work would be whether the specific LP can be solved more efficiently, either exactly or approximately, by using algorithms other than the standard ones (Ellipsoid or Interior Point Methods).

In order to solve the Linear Program~\eqref{eqn:utility-maximizing-LP}, one needs to know $\Prob[\DesV \sim \JDIST]{\IsBest[\DesV]{x}{k}}$ for all $i$ and $k$.
For some distributions \JDIST (e.g., Example~\ref{example:suboptimal-fairness-day}), these quantities can be calculated in closed form.
For others, they can be estimated using Monte Carlo sampling, as captured by the following proposition.

\begin{proposition} \label{prop:sampling}
Consider an algorithm that draws $m = \frac{(\kappa + 1) \log (2n)}{2 \epsilon^2}$ i.i.d.~samples of the agents' joint merits from \JDIST, and then estimates each probability $\Prob[\DesV \sim \JDIST]{\IsBest[\DesV]{x}{k}}$ by the empirical frequency with which $x$ was in position $k$ or higher. Then, with probability at least $1-n^{-\kappa}$, all $\Prob[\DesV \sim \JDIST]{\IsBest[\DesV]{x}{k}}$ are estimated with additive error at most $\pm \epsilon$.
\end{proposition}

\begin{proof}
  Focus on one agent $x$, and write $q_k = \Prob[\DesV \sim \JDIST]{\IsBest[\DesV]{x}{k}}$. Notice that the $q_k$ form the CDF of the rank of $x$.
  Let $Z_{k,j} = 1$ iff $x$ is among the top $k$ agents (by merit) in the \Kth{j} of the $m$ samples. Then, $\Prob{Z_{k,j} = 1} = q_k$, and the estimate $Z_k = \frac{1}{m} \cdot \sum_j Z_{k,j}$ is the average of $m$ independent $\text{Bin}(q_k)$ random variables. By the DKW Inequality for the uniform convergence of the empirical CDF to the true CDF \cite{dvoretzky:kiefer:wolfowitz:CDF,massart:DKW}, we get that with probability at least $1-2 \exp(-2m \epsilon^2) \geq 1-n^{-(\kappa+1)}$, all of the estimates $Z_k$ are within $\pm \epsilon$ of the true values $\Prob[\DesV \sim \JDIST]{\IsBest[\DesV]{x}{k}}$. A union bound over all $n$ agents now completes the proof.
\end{proof}

While the estimates may be off by additive $\epsilon$ terms, it is fairly easy to compensate for such errors at a small loss in fairness and utility, as follows:

\begin{proposition} \label{prop:LP-with-errors}
  For each $x,k$, let $q_{x,k}$ be an empirical estimate of $\Prob[\DesV \sim \JDIST]{\IsBest[\DesV]{x}{k}}$ such that $| q_{x,k} - \Prob[\DesV \sim \JDIST]{\IsBest[\DesV]{x}{k}} | \leq \epsilon$ and $\sum_x q_{x,k} = k$ for all $k$.
  Consider the solution to the LP~\eqref{eqn:utility-maximizing-LP} with fairness parameter $\FPAR$, using\footnote{Notice that the $q'_{x,k}$ in fact satisfy that $\sum_x q'_{x,k} = \frac{k}{k + n \epsilon} \sum_x (q_{x,k} + \epsilon) = \frac{k}{k + n \epsilon} \cdot (k + n \epsilon) = k$, so they can be used as input to the LP.}
  $q'_{x,k} = \frac{k(q_{x,k} + \epsilon)}{k + n \epsilon}$ in place of the (unknown) $\Prob[\DesV \sim \JDIST]{\IsBest[\DesV]{x}{k}}$.
  Then, the resulting sampling distribution is at least $(\frac{\FPAR}{1 + n \epsilon})$-fair, and guarantees the principal a utility within a factor $\frac{1}{1 + n \epsilon}$ of the optimum \FPAR-fair solution.
\end{proposition}

\begin{proof}
  First, notice that by the assumption that the $q_{x,k}$ were good approximations for $\Prob[\DesV \sim \JDIST]{\IsBest[\DesV]{x}{k}}$, we can bound that $q'_{x,k} \geq \frac{k}{k + n \epsilon} \cdot \Prob[\DesV \sim \JDIST]{\IsBest[\DesV]{x}{k}}$. 

  Because the LP's solution $(\InPos{x}{k})_{x,k}$ is \FPAR-fair with respect to the $q'_{x,k}$, we get that
  \[
    \sum_{k'=1}^k \InPos{x}{k'}
    \; \geq \; \FPAR \cdot q'_{x,k}
    \; \geq \; \frac{k\FPAR}{k + n \epsilon} \cdot \Prob[\DesV \sim \JDIST]{\IsBest[\DesV]{x}{k}}
    \; \geq \; \frac{\FPAR}{1 + n \epsilon} \cdot \Prob[\DesV \sim \JDIST]{\IsBest[\DesV]{x}{k}}
  \]
   for all $x, k$; thus, the solution is $(\frac{\FPAR}{1 + n \epsilon})$-fair.

  Next, we analyze the principal's utility. Let $(p^*_{x,k})_{x,k}$ be a \FPAR-fair solution maximizing the principal's utility, and write $z^*_{x,k} = \sum_{k'=1}^{k} p^*_{x,k'}$ for the probability that agent $x$ is ranked among the top $k$ positions in the optimum solution. Now define $z'_{x,k} = \min(z^*_{x,k}, k - \FPAR \cdot \sum_{x' \neq x} q'_{x',k})$.

  We will prove the following two facts: (1) The principal's utility under the probabilities $z'_{x,k}$ is not much smaller than under the original $z^*_{x,k}$, and (2) Every feasible solution $(p_{x,k})_{x,k}$ to the LP with fairness parameter \FPAR and $q'_{x,k}$ satisfies $\sum_{k' \leq k} p_{x,k'} \geq z'_{x,k}$ for all $x, k$.

\begin{enumerate}
\item To show the first claim, we first use a standard way to rewrite the principal's objective in terms of the $z^*_{x,k}$ (or $z'_{x,k}$), using the definition $z^*_{x,0} := z'_{x,0} := 0$:
  \begin{align}
    \sum_x \sum_{k=1}^n p^*_{x,k} \cdot \Expect[\DesV \sim \JDIST]{\Des{x}} \cdot\posweight{k}
& = \sum_x \Expect[\DesV \sim \JDIST]{\Des{x}} \cdot \sum_{k=1}^n (z^*_{x,k} - z^*_{x,k-1}) \cdot \posweight{k} \nonumber
\\ & = \sum_x \Expect[\DesV \sim \JDIST]{\Des{x}} \cdot \left( \sum_{k=1}^n z^*_{x,k} \cdot  \posweight{k} - \sum_{k=0}^{n-1} z^*_{x,k} \cdot \posweight{k+1} \right) \nonumber
\\ & = \sum_x \Expect[\DesV \sim \JDIST]{\Des{x}} \cdot \left( \posweight{n} + \sum_{k=1}^{n-1} z^*_{x,k} \cdot (\posweight{k} - \posweight{k+1}) \right). \label{eqn:rewritten-utility}
\end{align}

  Because $z'_{x,k} \leq z^*_{x,k+1}$ for all $x, k$, writing $p'_{x,k} := z'_{x,k} - z'_{x,k-1}$, we can also express the principal's utility under $(z'_{x,k})_{x,k}$ in the same way, simply replacing the terms $z^*_{x,k}$ with $z'_{x,k}$ in \eqref{eqn:rewritten-utility}. Note that the $p'_{x,k}$ do not form a valid solution to the LP, because the ``probabilities'' do not necessarily sum up to 1 each across agents or across positions. However, we are only using this ``solution'' to help with our bounds, and feasibility is not required.

  We can write the principal's loss in utility going from $z^*_{x,k}$ to $z'_{x,k}$ as follows:
  \begin{align}
& \phantom{=} \sum_x \Expect[\DesV \sim \JDIST]{\Des{x}} \cdot \left( \posweight{n}
    + \sum_{k=1}^{n-1} z^*_{x,k} \cdot (\posweight{k} - \posweight{k+1}) \right)
- \sum_x \Expect[\DesV \sim \JDIST]{\Des{x}} \cdot \left( \posweight{n}
    + \sum_{k=1}^{n-1} z'_{x,k} \cdot (\posweight{k} - \posweight{k+1}) \right) \nonumber
\\ & = \sum_x \Expect[\DesV \sim \JDIST]{\Des{x}} \cdot \sum_{k=1}^{n-1} (z^*_{x,k} - z'_{x,k}) \cdot (\posweight{k} - \posweight{k+1}) \nonumber
\\ & = \sum_{k=1}^{n-1} (\posweight{k} - \posweight{k+1}) \cdot \sum_x \Expect[\DesV \sim \JDIST]{\Des{x}} \cdot (z^*_{x,k} - z'_{x,k}). \label{eqn:utility-difference}
\end{align}

Notice that $\posweight{k} - \posweight{k+1} \geq 0$ for all $k$, and $\Expect[\DesV \sim \JDIST]{\Des{x}} \geq 0$ for all $x$. To upper-bound the loss in utility, we therefore can apply bounds for each of the terms $z^*_{x,k} - z'_{x,k}$.
Focus on one particular pair $x,k$. Notice that the LP constraints (specifically, the third constraint and the first constraint) imply that

\[
  z^*_{x,k}
  \; = \; k - \sum_{x' \neq x} z^*_{x',k}
  \; \leq \; k - \FPAR \cdot \sum_{x' \neq x} \Prob[\DesV \sim \JDIST]{\IsBest[\DesV]{x'}{k}}
  \; = \; k - \FPAR \cdot (k - \Prob[\DesV \sim \JDIST]{\IsBest[\DesV]{x}{k}}).
\]

  If $z'_{x,k} < z^*_{x,k}$, then
\[
   z'_{x,k}
   \; = \; k - \FPAR \cdot \sum_{x' \neq x} q'_{x',k}
   \; = \; k - \FPAR \cdot \sum_{x' \neq x} \frac{k(q_{x,k} + \epsilon)}{k + n \epsilon}
   \; = \; k - \frac{k\FPAR}{k + n \epsilon} \cdot (k-q_{x,k} + (n-1) \epsilon).
\]

  Therefore, the difference is at most
  \begin{align*}
    z^*_{x,k} - z'_{x,k}
    & \leq \frac{k\FPAR}{k + n \epsilon} \cdot (k-q_{x,k} + (n-1) \epsilon)
- \FPAR \cdot (k - \Prob[\DesV \sim \JDIST]{\IsBest[\DesV]{x}{k}})
    \\  & = \frac{\FPAR}{k + n \epsilon} \cdot \left(
          (k^2-kq_{x,k} + k(n-1) \epsilon)
- (k^2 + kn\epsilon - (k + n\epsilon) \cdot \Prob[\DesV \sim \JDIST]{\IsBest[\DesV]{x}{k}}) \right)
    \\  & \stackrel{(*)}{\leq} \frac{\FPAR}{k + n \epsilon} \cdot \left(
          (-k (\Prob[\DesV \sim \JDIST]{\IsBest[\DesV]{x}{k}} - \epsilon) - k \epsilon)
+ (k + n\epsilon) \cdot \Prob[\DesV \sim \JDIST]{\IsBest[\DesV]{x}{k}}) \right)
    \\  & = \frac{\FPAR n \epsilon}{k + n \epsilon} \cdot
             \Prob[\DesV \sim \JDIST]{\IsBest[\DesV]{x}{k}}
    \\  & \stackrel{(**)}{\leq} \frac{n \epsilon}{k + n \epsilon} \cdot z^*_{x,k}.
  \end{align*}
  Here, the line labeled (*) used that the $q_{x,k}$ approximate the true probabilities $\Prob[\DesV \sim \JDIST]{\IsBest[\DesV]{x}{k}}$ to within additive error at most $\epsilon$, and the line labeled (**) used that the $z^*_{x,k}$ formed a \FPAR-fair solution.\footnote{For $\FPAR=0$, the calculations do not apply, but in that case, the algorithm can completely ignore the estimated probabilities, and will obtain the optimum solution.}

  We now substitute this bound (for each $k,x$) into \eqref{eqn:utility-difference}, obtaining that the principal's loss in utility is at most

  \begin{align*}
  \sum_{k=1}^{n-1} (\posweight{k} - \posweight{k+1}) \cdot \sum_x \Expect[\DesV \sim \JDIST]{\Des{x}} \cdot \frac{n \epsilon}{k + n \epsilon} \cdot z^*_{x,k}
& \leq \frac{n \epsilon}{1 + n \epsilon} \sum_{k=1}^{n-1} (\posweight{k} - \posweight{k+1}) \cdot \sum_x \Expect[\DesV \sim \JDIST]{\Des{x}} \cdot z^*_{x,k},
  \end{align*}
  which is exactly $\frac{n \epsilon}{1 + n\epsilon}$ times the principal's utility under the solution $z^*_{x,k}$, i.e., the optimal utility.
  Thus, the utility obtained from using the approximate values is within at least a factor $1 - \frac{n \epsilon}{1 + n\epsilon} = \frac{1}{1 + n\epsilon}$ of optimal.

\item Next, we show that every feasible solution $(\InPos{x}{k})_{x,k}$ to the LP with fairness parameter \FPAR and $q'_{x,k}$ satisfies $\sum_{k' \leq k} \InPos{x}{k'} \geq z'_{x,k}$ for all $x, k$. In fact, we show that $\sum_{k' \leq k} \InPos{x}{k'} \geq k - \FPAR \cdot \sum_{x' \neq x} q'_{x',k}$, which in turn is at least $z'_{x,k}$ by definition of $z'_{x,k}$.

  To see this, note that for any feasible solution and for all $x,k$, the fairness constraint implies that $\sum_{k'=1}^k \InPos{x}{k'} \geq \FPAR \cdot q'_{x,k}$ and furthermore, $\sum_x \InPos{x}{k'} = 1$ for all $k'$. Therefore, for any fixed $x, k$,
\[
  k
  \; = \; \sum_{x'} \sum_{k'=1}^k \InPos{x'}{k'}
  \; = \; \sum_{k'=1}^k \InPos{x}{k'} + \sum_{x' \neq x} \sum_{k'=1}^k \InPos{x'}{k'}
  \; \geq \; \sum_{k'=1}^k \InPos{x}{k'} + \sum_{x' \neq x} \FPAR \cdot q'_{x',k}.
\]

  Rearranging this inequality gives us the claimed bound.
\end{enumerate}

Now consider the optimal solution $\InPos{x}{k}$ (maximizing the principal's utility) with fairness parameter \FPAR and estimated probabilities $q'_{x,k}$.
For each $x, k$, define $z_{x,k} = \sum_{k'=1}^k \InPos{x}{k'}$.
Then, $z_{x,k} \geq z'_{x,k}$ for all $x, k$, and the utility under $(\InPos{x}{k})_{x,k}$ is given by \eqref{eqn:rewritten-utility} (with $z_{x,k}$ in place of $z^*_{x,k}$). In particular, it is at least as large as under $(z'_{x,k})_{x,k}$, and thus within a factor of $\frac{1}{1+n\epsilon}$ of the optimum.
\end{proof}

By Proposition~\ref{prop:LP-with-errors}, if the principal wants to approximate fairness and utility to within a factor $1-\epsilon$, it suffices to approximate the $\Prob[\DesV \sim \JDIST]{\IsBest[\DesV]{x}{k}}$ to within an additive error of at most $\frac{\epsilon}{n(1-\epsilon)}$. In turn, by Proposition~\ref{prop:sampling}, it is sufficient to draw $O(\frac{\kappa n^2 \log n}{2 \epsilon^2})$ samples from \JDIST to achieve this approximation with probability at least $1-n^{-\kappa}$; in particular, the number is polynomial in $n$ and $1/\epsilon$.

%% file: TS-not-optimal.tex
\begin{example} \label{example:suboptimal-fairness-day}
Consider $n=3$ agents, namely $a, b$, and $c$.
Under \JDIST, their merits $\Des{a} = 1$, $\Des{b} \sim \Bernoulli{\nicefrac{1}{2}}$, and $\Des{c} \sim \Bernoulli{\nicefrac{1}{2}}$ are drawn independently.%
\footnote{Technically, this distribution violates the assumption of non-identical merit of agents under \JDIST. This is easily remedied by adding --- say --- i.i.d.~$\mathcal{N}(0,\epsilon)$ Gaussian noise to all \Des{i}, with very small $\epsilon$. We omit this detail since it is immaterial and would unnecessarily overload notation.}
  The position weights are $\posweight{1} = 1, \posweight{2} = 1$, and $\posweight{3}=0$. 
  
  Now, since $\posweight{1}=\posweight{2}=1$ and agents $b$ and $c$ are i.i.d., any policy that always places agent $a$ in positions 1 or 2 is optimal.
  In particular, this is true for the policy \probrankOPT which chooses uniformly at random from among $\detrankOPT[1] = \langle a,b,c \rangle,\: \detrankOPT[2] = \langle a,c,b \rangle,\: \detrankOPT[3] = \langle b,a,c \rangle$, and $\detrankOPT[4] = \langle c,a,b \rangle$.

For the specific distribution \JDIST, assuming uniformly random tie breaking, we can calculate the probabilities $\Prob[\DesV \sim \JDIST]{\IsBest[\DesV]{x}{k}}$ in closed form:
\begin{align*}
  \left(\Prob[\DesV \sim \JDIST]{\IsBest[\DesV]{x}{k}}\right)_{x,k}
  & = \nicefrac{1}{24} \cdot \begin{pmatrix}
    14 & 22 & 24 \\
    5 & 13 & 24 \\
    5 & 13 & 24
    \end{pmatrix}.
\end{align*} 

The probability of $a,\:b,\:c$ being \emph{placed} in the top $k$ positions by \probrankOPT can be calculated as follows:
\begin{align*}
  \InPosMat[\probrankOPT]
  & = \nicefrac{1}{24} \cdot \begin{pmatrix}
    12 & 24 & 24 \\
    6 & 12 & 24 \\
    6 & 12 & 24
    \end{pmatrix}.
\end{align*}

In particular, this implies that \probrankOPT is \FPAR-fair for every $\FPAR \leq \nicefrac{12}{14} = \nicefrac{6}{7}$. This bound can be pushed up by slightly increasing the probability of ranking agent $a$ at position 1 (hence increasing fairness to agent $a$ in position 1 at the expense of agents $b$ and $c$ in positions 1--2). Figure~\ref{fig:fairness_day_vs_lp} shows the principal's optimal utility for different fairness parameters \FPAR, derived from the LP~\eqref{eqn:utility-maximizing-LP}. This optimal utility is contrasted with the utility of \probrankFD{\FPAR}, which is the convex combination of the utilities of \probrankOPT and \probrankTS, by Lemma~\ref{lemma:mixture}.

\begin{figure}[htb]
\centering
    \includegraphics[width=0.7\textwidth]{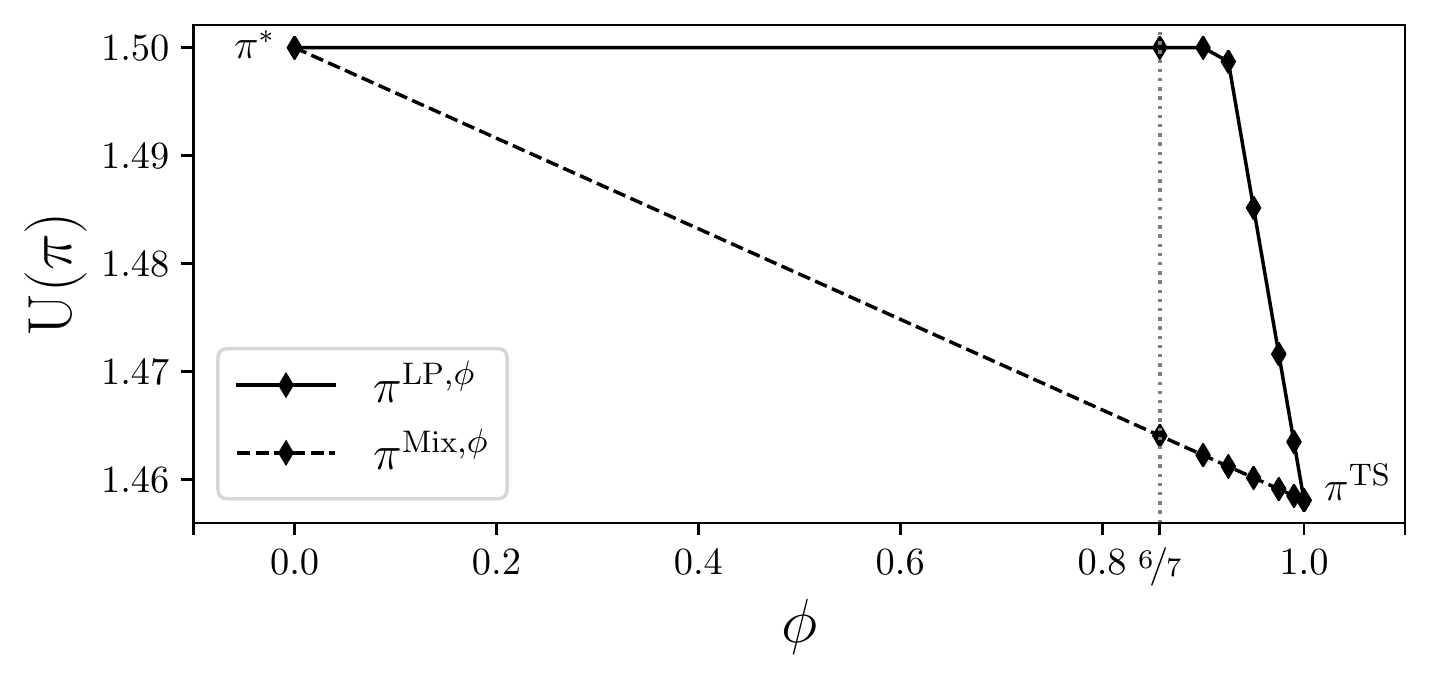}
    \vspace*{-0.15cm}
    \caption{Utility of $\probrankFD{\FPAR}$ and $\probrankLP{\FPAR}$ for Example~\ref{example:suboptimal-fairness-day} as one varies \FPAR.}
    \label{fig:fairness_day_vs_lp}
     \vspace*{-0.15cm}
\end{figure}

\end{example}

%% file: experiments.tex
To evaluate our approach in a recommendation setting with a realistic preference distribution, we designed the following experimental setup based on the MovieLens 100K (ML-100K) dataset.
The dataset contains 100,000 ratings, by 600 users, on 9,000 movies belonging to 18 genres \cite{harper2015movielens}.
In our setup, for each user, the principal is a recommender system that has to generate a ranking of movies $\MovieSet_g$ for one of the genres $g$ (e.g., Horror, Romance, Comedy) according to a notion of \emph{merit} of the movies we define as follows.\footnote{The code to reproduce the experimental evaluation is available at \url{https://github.com/ashudeep/ranking-fairness-uncertainty}}

\subsection{Modeling the Merit Distribution}
We define the (unknown) merit \Des{m} of a movie $m$ as the average rating of the movie across the user population\footnote{For a personalized ranking application, an alternative would be to choose each user's (mostly unknown) rating as the merit criterion instead of the average rating across the user population.} --- this merit is unknown because most users have not seen/rated most movies. 
To be able to estimate this merit based on ratings in the ML-100K dataset, and to concretely define its underlying distribution and the corresponding fairness criteria, we define a generative model of user ratings. The model assumes that each rating of a movie $m \in \MovieSet_g$ is drawn from a multinomial distribution over $\{1, 2, 3, 4, 5\}$ with (unknown) parameters $\multparamsVec[m] = (\multparams[m]{1}, \ldots, \multparams[m]{5})$.

\begin{enumerate}[leftmargin=0cm,itemindent=.5cm,labelwidth=\itemindent,labelsep=0cm,align=left]
\item[\textbf{Prior}: ] These parameters themselves follow a Dirichlet prior
$\multparamsVec[m] \sim \Dir{\dirichletparams}$
with known parameters $\dirichletparams = (\dpar{1}, \dpar{2}, \dpar{3}, \dpar{4}, \dpar{5})$.
We assume that the parameters of the Dirichlet prior are of the form $\dpar{r} = \scaleparam \cdot p_r$ where $\scaleparam$ is a scaling factor and $p_r = \ProbC{\Rating=r}{\Dataset}$ denotes the marginal probability of observing the rating $r$ in the full MovieLens dataset. 

The scaling factor $\scaleparam$ determines the weight of the prior compared to the observed data since it acts as a pseudo-count in $\dirichletparams'$ below. For the sake of simplicity, we use $\scaleparam=\scaleparamValue$ in the following for all movies and genres.
\item[\textbf{Posterior}: ]
Since the Dirichlet distribution is the conjugate prior of the multinomial distribution, the posterior distribution based on the ratings observed in the dataset \Dataset is also a Dirichlet distribution, but with parameters $\dirichletparams' = \dirichletparams+\countratings{m} = (\dpar{1}+\countone{m}{1},\: \ldots,\: \dpar{5}+\countone{m}{5})$ where \countone{m}{r} is the number of ratings of $r$ for the movie $m$ in the dataset \Dataset. 

\end{enumerate}

\subsection{Ranking Policies}

\begin{enumerate}[leftmargin=0.0cm,itemindent=0.25cm,labelwidth=\itemindent,labelsep=0cm,align=left]
\item[\textbf{Utility-Maximizing Ranking} (\probrankOPT): ] We use the DCG function \cite{burges2005learning} with position weights $\posweight{k} = \nicefrac{1}{\log_2(1+k)}$ as our utility measure. Since the weights are indeed strictly decreasing, as described in Section~\ref{sec:policies}, the optimal ranking policy \probrankOPT sorts the movies (for the particular query) by decreasing expected merit, which is the expected average rating
\ExpRelevance{m} under the posterior Dirichlet distribution, and can be computed in closed form as follows:
\begin{align}
  \ExpRelevance{m} & \triangleq \Expect[\multparamsVec{} \sim \ProbC{\multparamsVec[m]}{\Dataset}]{\Des{m}(\multparamsVec{})}
                \; = \; \sum_{r=1}^5 r\cdot \frac{\dpar{r} + \countone{m}{r}}{\sum_{r'} \dpar{r'}+ \countone{m}{r'}}.
\label{eq:average-rating-full}
\end{align}

\item[\textbf{Fully Fair Ranking Policy} (\probrankTS): ]
A fair ranking, in this case, ensures that, for all positions $k$, a movie is placed in the top $k$ positions according to the posterior merit distribution. In this setup, a fully fair ranking policy \probrankTS is obtained by sampling the multinomial parameters $\multparamsVec[m]$ for each movie $m \in \MovieSet_g$ and computing $\Des{m}(\multparamsVec[m])$ to rank them:
\[\probrankTS(\MovieSet_g) \sim \argsort_m \Des{m}(\multparamsVec[m]) \text{ s.t. } \multparamsVec[m]\sim \Prob{\multparamsVec[m] | \Dataset}.\]

\item[\textbf{\FD Ranking Policy} (\probrankFD{\FPAR}): ] The policies \probrankFD{\FPAR} randomize between the fully fair and utility-maximizing ranking policies with probabilities \FPAR and $1-\FPAR$, respectively. 

\item[\textbf{LP Ranking Policy (\probrankLP{\FPAR})}: ]The \FPAR-fair policies $\probrankLP{\FPAR}$ require the principal to have access to the probabilities $\Prob[\DesV \sim \JDIST]{\IsBest[\DesV]{m}{k}}$ which we estimate using \numMCSamples Monte Carlo samples, so that any estimation error becomes negligible. 
\end{enumerate}
\begin{figure}
    \centering
    \includegraphics[width=0.75\textwidth]{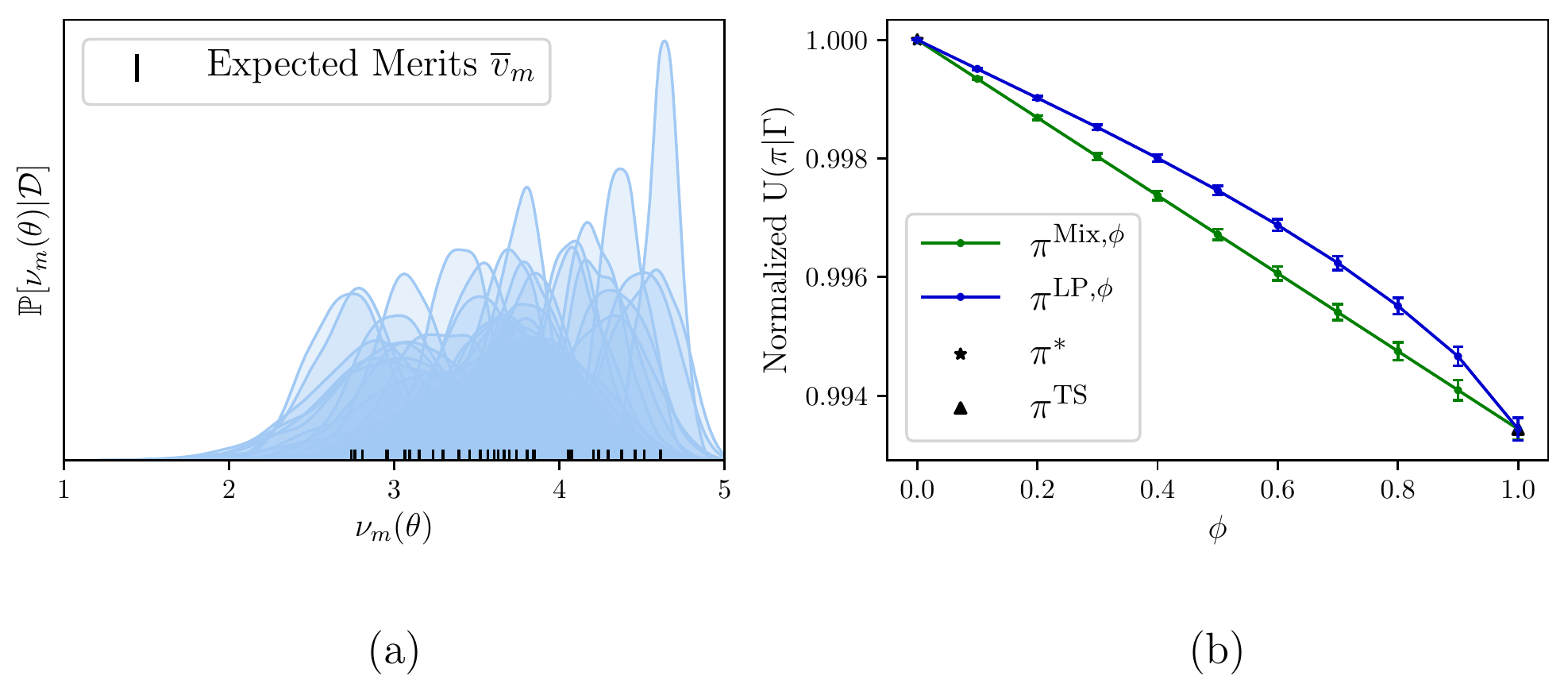}
    \vspace{-0.2cm}
    \caption{(a) Posterior distribution of ratings (merit) for a subset of \selectedgenre movies, (b) Tradeoff between Utility and Fairness, as captured by $\FPAR$.}
    \vspace{-0.25cm}
    \label{fig:movielens}
\end{figure}

\subsection{Observations and Results}

In the experiments presented, we used the ranking policies \probrankOPT, \probrankTS, \probrankFD{\FPAR} and \probrankLP{\FPAR} to create separate rankings for each of the 18 genres. For each genre, the task is to rank a random subset of \numMoviesPerGenre movies from that genre. To get a posterior with an interesting degree of uncertainty, we take a \percentSampleDataset i.i.d.~samples from $\Dataset$ to infer the posterior for each movie. We observe that the results are qualitatively consistent across genres, and we thus focus on detailed results for the genre \selectedgenre as a representative example. Its posterior merit distribution over a subset is visualized in Figure~\ref{fig:movielens}(a). Note that substantial overlap exists between the marginal merit distributions of the movies, indicating that as opposed to \probrankOPT (which sorts based on the expected merits), the policy \probrankTS will randomize over many different rankings. 

\paragraph{Observation 1:} We evaluate the cost of fairness to the principal in terms of loss in utility, as well as the ability of \probrankLP{\FPAR} to minimize this cost for \FPAR-fair rankings. Figure~\ref{fig:movielens}(b) shows this cost in terms of expected Normalized DCG (i.e., $\text{NDCG}=\nicefrac{\text{DCG}}{\max(\text{DCG})}$ as in \cite{jarvelin2002cumulated}). These results are averaged over \numruns runs with different subsets of movies and different training samples. The leftmost end corresponds to the NDCG of \probrankOPT, while the rightmost point corresponds to the NDCG of the 1-fair policy \probrankTS. 

The drop in NDCG is below one percentage point, which is consistent with the results for the other genres. We also conducted experiments with other values of $s$, data set sizes, and choices of $\posweight{k}$; even under the most extreme conditions, the drop was at most 2 percent. While this rather small drop may be surprising at first, we point out that uncertainty in the estimates affects the utility of both \probrankOPT and \probrankTS. By industry standards, a 2\% drop in NDCG is considered quite substantial; however, it is not catastrophic and hence bodes well for possible adoption.

\paragraph{Observation 2:}
Figure~\ref{fig:movielens}(b) also compares the trade-off in NDCG in response to the fairness approximation parameter \FPAR\ for both \probrankFD{\FPAR} and \probrankLP{\FPAR}. We observe that the utility-optimal policy \probrankLP{\FPAR} provides gains over \probrankFD{\FPAR}, especially for large values of \FPAR. Thus, using \probrankLP{\FPAR} can further reduce the cost of fairness discussed above. 

\begin{figure*}[!htp] 
    \centering
    \includegraphics[width=\textwidth]{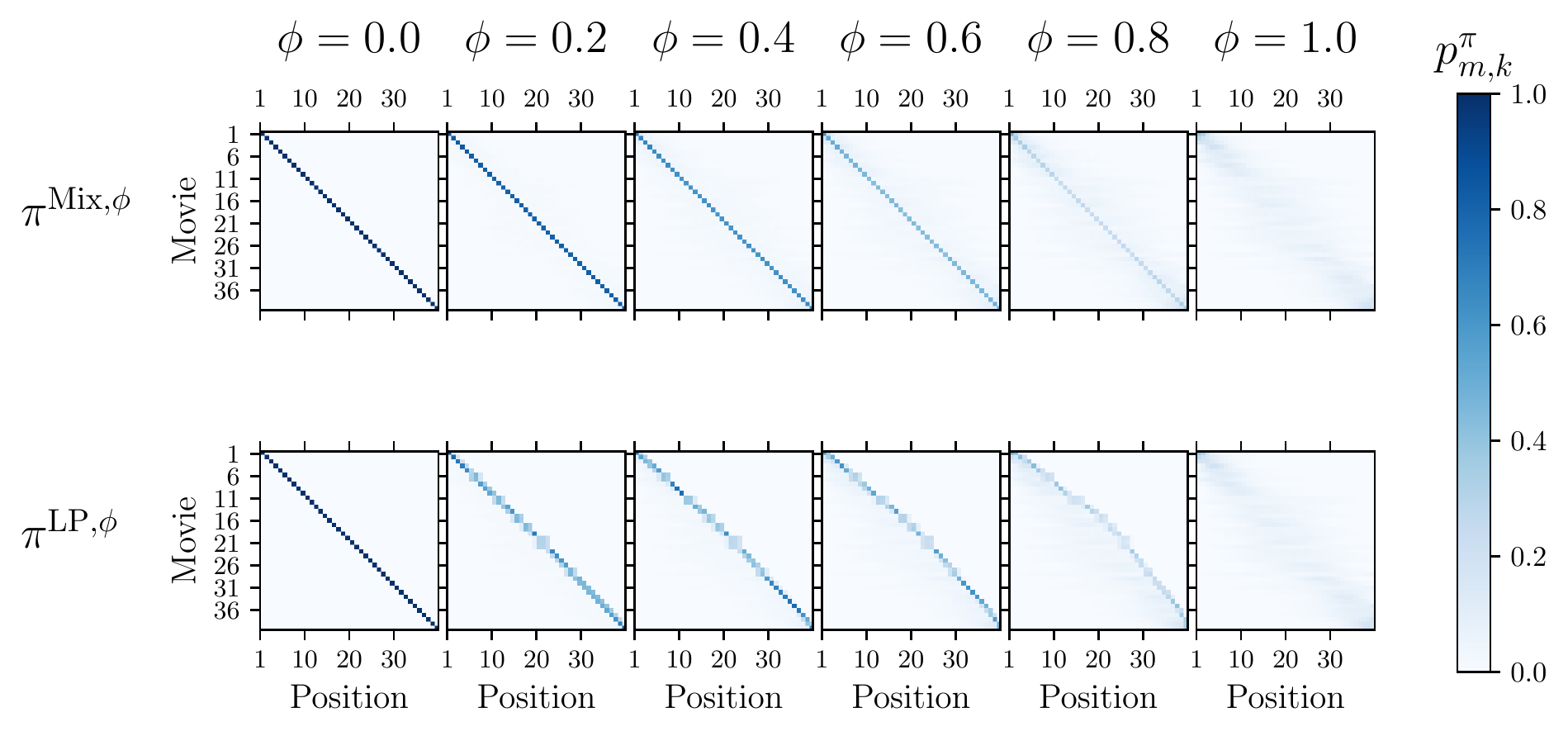}
    \caption{Comparison of marginal rank distribution matrices for $\probrankFD{\FPAR}$ and $\probrankLP{\FPAR}$ on \selectedgenre movies.}
    \label{fig:movielens_ranking_distribution}
\end{figure*}
\paragraph{Observation 3:}
To provide intuition about the difference between \probrankFD{\FPAR} and \probrankLP{\FPAR}, Figure~\ref{fig:movielens_ranking_distribution} visualizes the marginal rank distributions $\InPos{m}{k}$, i.e., the probability that movie $m$ is ranked at position $k$. The key distinction is that, for intermediate values of \FPAR, \probrankLP{\FPAR} exploits a non-linear structure in the ranking distribution (to achieve a better trade-off) while \probrankFD{\FPAR} merely interpolates linearly between the solutions for $\FPAR=0$ and $\FPAR=1$.

Based on these observations, in general, the utility loss due to fairness is small, and can be further reduced by optimizing the ranking distribution with the LP-based approach.
These results are based on the definition of merit as the average rating of movies over the entire user population. A more realistic setting would personalize rankings for each user, with merit defined as the expected relevance of a movie to the user. In our experiments, the results under such a setup were quite similar, and are hence omitted for brevity and clearer illustration.

%% file: real_world_experiments.tex
To study the effect of deploying a fair ranking policy in a real ranking system, we built and fielded a paper recommendation system at the 2020 ACM SIGKDD Conference on Knowledge Discovery and Data Mining. 
The goal of the experiment is to understand the impact of fairness under real user behavior, as opposed to simulated user behavior that is subject to modeling assumptions.
Specifically, we seek to answer two questions: (a) Does a fair ranking policy lead to a more equitable distribution of exposure among the papers? (b) Does fairness substantially reduce the utility of the system to the users?

The users of the paper recommendation system were the participants of the conference, which was held virtually in 2020. Signup and usage of the system was voluntary and subject to informed consent. Each user was recommended a personalized ranking of the papers published at the conference. This ranking was produced either by \detrankOPT or by \probrankTS, and the assignment of users to treatment (\probrankTS) or control (\detrankOPT) was randomized.

\subsection{Modeling the Merit Distribution}
The merit of a paper for a particular user is based in part on a relevance score $\Similarity{u}{i}$ that relates features of the user (e.g., bag-of-words representation of recent publications, co-authorship) to features of each conference paper (e.g., bag-of-words representation of paper, citations). Most prominently, the relevance score $\Similarity{u}{i}$ contains the TFIDF-weighted cosine similarity between the bag-of-words representations.

We model the uncertainty in $\Similarity{u}{i}$ with regard to the true relevance as follows. First, we observe that all papers were accepted to the conference and thus must have been deemed relevant to at least some fraction of the audience by the peer reviewers. This implies that papers with uniformly low $\Similarity{u}{i}$ across all/most participants are not irrelevant; we merely have high uncertainty as to which participants the papers are relevant to. For example, papers introducing new research directions or bringing in novel techniques may have uniformly low scores $\Similarity{u}{i}$ under the bag-of-words model that is less certain about who wants to read these papers compared to papers in established areas.
To formalize uncertainty, we make the assumption that a paper's relevance to a user follows a normal distribution centered at $\Similarity{u}{i}$, and with standard deviation equal to $\sd{i}$ (dependent only on the paper, not the user) such that $\max_u \Similarity{u}{i} + \gamma \cdot \sd{i} = 1 + \epsilon$. (For our experiments, we chose $\epsilon = 0.1$ and $\gamma = 2$.) This choice of $\sd{i}$ ensures that there exists at least one user $u$ such that the (sampled) relevance score $\SimilaritySample{u}{i}$ is greater than 1 with some significant probability; more specifically, we ensure that the probability of having relevance $1+\epsilon$ is at least as large as that of exceeding the mean by two standard deviations. Furthermore, $\epsilon>0$ ensures that all papers have a non-deterministic relevance distribution, even papers with $\max_u \Similarity{u}{i} = 1$.

\subsection{Ranking Policies}

Users in the control group \controlusers\ received rankings in decreasing order of $\Similarity{u}{i}$.
Users in the treatment group \treatmentusers\ received rankings from the fair policy that sampled scores from the uncertainty distribution, $\SimilaritySample{u}{i} \sim \mathcal{N}(\Similarity{u}{i}, \sd{i})$, and ranked the papers by decreasing \SimilaritySample{u}{i}.

\subsection{Results and Observations}

\begin{figure}
\begin{floatrow}
\ffigbox[0.45\textwidth]{
\includegraphics[width=0.45\textwidth]{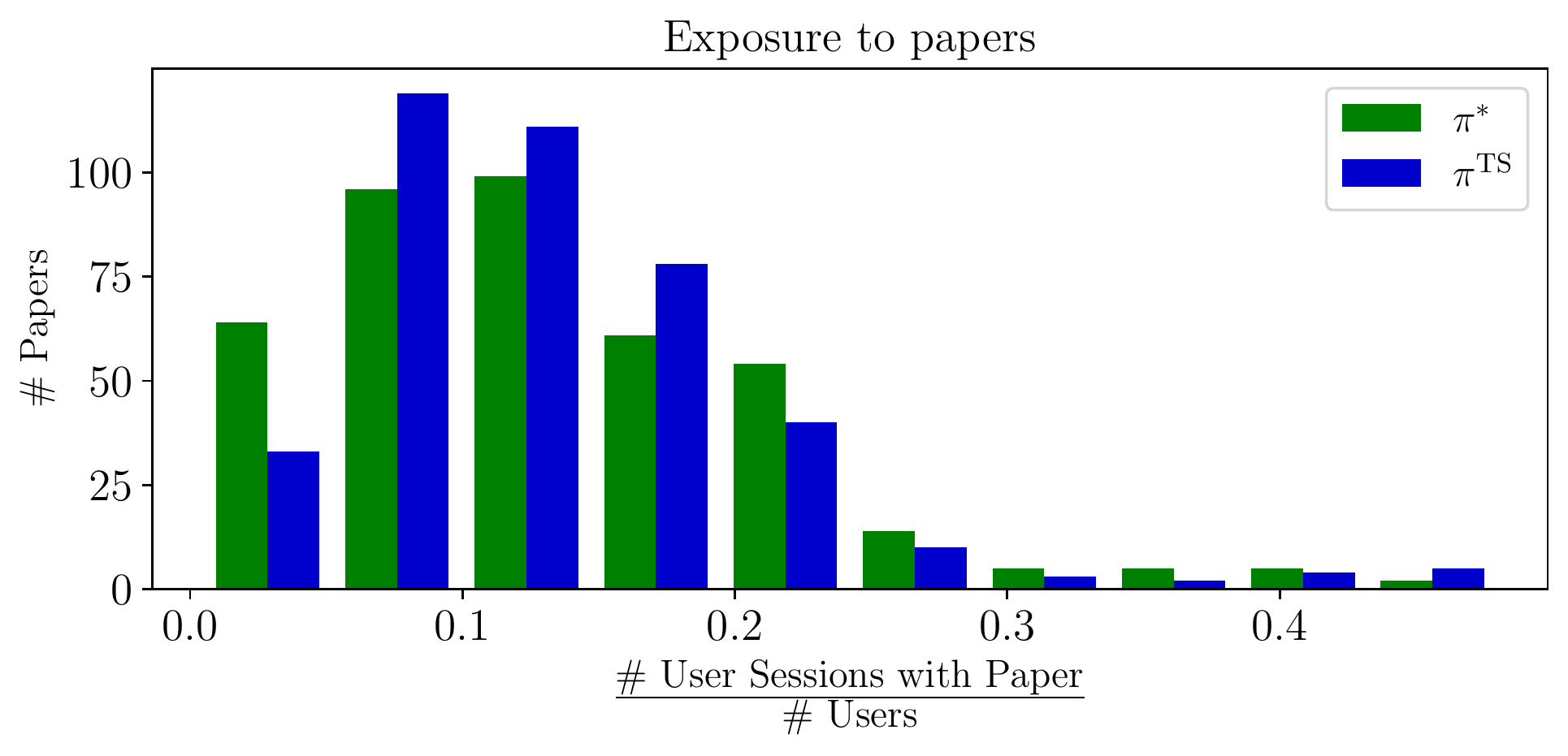}
}{%
  \caption{Distribution of the exposure of papers in treatment and control.\label{fig:exposure-engagement-kdd}}
}
\capbtabbox[0.55\textwidth]{%
\resizebox{0.55\textwidth}{!}
{\begin{tabular}{@{}lcccc@{}}
\toprule
                         & \multicolumn{2}{l}{\begin{tabular}[c]{@{}l@{}}Number of Users\\with activity\end{tabular}} & \multicolumn{2}{l}{\begin{tabular}[c]{@{}l@{}}Average Activity\\ Per User\end{tabular}} \\ \cmidrule(l){2-5} 
                         & $\probrankOPT$                          & $\probrankTS$                          & $\probrankOPT$         & $\probrankTS$        \\ \midrule
Total Number of users    & 213                                     & 248                                    &       -                 & -                     \\
Num.~of pages examined & -                                        & -                                       & 10.8075                & 10.7984              \\
\texttt{Read Abstract}           & 92                                      & 101                                    & 3.7230                 & 2.6774               \\
\texttt{Add to Calendar}          & 51                                      & 50                                     & 1.4366                 & 0.8508               \\
\texttt{Read PDF}                 & 40                                      & 52                                     & 0.5258                 & 0.5323               \\
\texttt{Add Bookmark}             & 16                                      & 13                                     & 0.3192                 & 0.6129               \\ \bottomrule
\end{tabular}%
}}
{%
  \caption{User engagement under the two conditions \probrankOPT and \probrankTS. None of the differences are statistically significant. (For user actions, this is specifically due to the small sample size).}%
  \label{table:real-world}
}
\end{floatrow}
\end{figure}

We first analyze if the fair policy provided more equitable exposure to the papers. In this real-world evaluation, exposure is not equivalent to rank, but depends on whether users actually browsed to a given position in the ranking. Users could browse their ranking in pages of 5 recommendations each; we count a paper as \emph{exposed} if the user scrolled to its page. 
\begin{enumerate}[leftmargin=0cm,itemindent=.5cm,labelwidth=\itemindent,labelsep=0cm,align=left]
    \item[\textbf{Observation 1: }] Figure~\ref{fig:exposure-engagement-kdd} compares the histograms of exposure of the papers in the treatment and control groups. Under the fair policy, the number of papers in the lowest tier of exposure is roughly halved compared to the control condition. This verifies that the fair policy does have an impact on exposure in a real-world setting, and it aligns with our motivation that a fair ranking policy distributes exposure more equally among the ranked agents. This is also supported by comparing the Gini inequality coefficient \cite{gini1936measure} for the two distributions: $G(\probrankOPT) = 0.3302$, while $G(\probrankTS) = 0.2996$ (where a smaller coefficient means less inequality in the distribution).
    \item[\textbf{Observation 2: }] To evaluate the impact of fairness on user engagement, we analyze a range of engagement metrics as summarized in Table~\ref{table:real-world}. While a total of 923 users signed up for the system ahead of the conference (and were randomized into treatment and control groups), 462 never came to the system after all. Of the users that came, 213 users were in \controlusers, and 248 users were in \treatmentusers. Note that this difference is not caused by the treatment assignment, since users had no information about their assignment/ranking before entering the system.

The first engagement metric we computed is the average number of pages that users viewed under both conditions. With roughly 10.8 pages (about 54 papers), engagement under both conditions was almost identical. Users also had other options to engage, but there is no clear difference between the conditions, either. On average, they read more paper abstracts and added more papers to their calendar under the control condition, but read more PDF and added more bookmarks under the treatment condition. However, none of these differences is significant at the 95\% level for either a Mann-Whitney U test or a two-sample t-test. While the sample size is small, these findings align with the findings on the synthetic data, namely that fairness did not appear to place a large cost on the principal (here representing the users).
\end{enumerate}

%% file: conclusions.tex
We believe that the focus on uncertainty we proposed in this paper constitutes a principled approach to capturing the intuitive notion of fairness to agents with similar features: rather than focusing on the features themselves, the key insight is that the features' similarity entails significant statistical uncertainty about which agent has more merit. Randomization provides a way to fight fire with fire, and axiomatize fairness in the presence of such uncertainty.

Our work raises a wealth of questions for future work. Perhaps most importantly, as discussed in Section~\ref{sec:discussion}, to operationalize our proposed notion of fairness, it is important to derive principled merit distributions \JDIST based on the observed features.
Our experiments were based on ``reasonable'' notions of merit distributions and concluded that fairness might not have to be very expensive to achieve for a principal.
However, much more experimental work is needed to truly evaluate the impact of fair policies on the utility that is achieved. It would be particularly intriguing to investigate which types of real-world settings lend themselves to implementing fairness at little cost, and which force a steep trade-off between the two objectives.

Our work also raises several interesting theoretical questions. In Section~\ref{sec:discussion}, we show one setting in which forcing the principal to use a fair policy drastically increases the principal's incentives to form a more accurate posterior \JDIST for a minority group. We did not prove a general result in this vein. We ask: will the incentives of a principal to learn a better posterior \JDIST always (weakly) increase if the principal is forced to be fairer? If true, this would provide a fascinating additional benefit of fairness requirements.

Another interesting question concerns the utility loss incurred by using the policy \FD. As shown in Section~\ref{sec:fairness-day}, \FD is in general not optimal. However, in the example from Section~\ref{sec:fairness-day} as well as in our experiments in Section~\ref{sec:experiments}, the loss in utility was quite small. An interesting question would be to bound the worst-case loss in the utility of \FD, compared to the LP-based policy. In particular, this question is of interest due to the simplicity of the \FD policy; it does not require the computationally expensive solution of an LP or an explicit estimate of marginal rank probabilities under \JDIST.

%% file: main.bbl
\begin{thebibliography}{56}
\providecommand{\natexlab}[1]{#1}
\providecommand{\url}[1]{\texttt{#1}}
\expandafter\ifx\csname urlstyle\endcsname\relax
  \providecommand{\doi}[1]{doi: #1}\else
  \providecommand{\doi}{doi: \begingroup \urlstyle{rm}\Url}\fi

\bibitem[Asudehy et~al.(2019)Asudehy, Jagadishy, Stoyanovichz, and
  Das]{asudehy2017designing}
Abolfazl Asudehy, HV~Jagadishy, Julia Stoyanovichz, and Gautam Das.
\newblock Designing fair ranking schemes.
\newblock \emph{SIGMOD}, 2019.

\bibitem[Barocas and Selbst(2016)]{barocas2016big}
Solon Barocas and Andrew~D Selbst.
\newblock Big data's disparate impact.
\newblock \emph{Cal. L. Rev.}, 2016.

\bibitem[Beutel et~al.(2019)Beutel, Chen, Doshi, Qian, Wei, Wu, Heldt, Zhao,
  Hong, Chi, et~al.]{beutel2019fairness}
Alex Beutel, Jilin Chen, Tulsee Doshi, Hai Qian, Li~Wei, Yi~Wu, Lukasz Heldt,
  Zhe Zhao, Lichan Hong, Ed~H Chi, et~al.
\newblock Fairness in recommendation ranking through pairwise comparisons.
\newblock In \emph{KDD}, 2019.

\bibitem[Biega et~al.(2018)Biega, Gummadi, and Weikum]{biega2018equity}
Asia~J Biega, Krishna~P Gummadi, and Gerhard Weikum.
\newblock Equity of attention: Amortizing individual fairness in rankings.
\newblock In \emph{SIGIR}, 2018.

\bibitem[Birkhoff(1946)]{birkhoff1946tres}
Garrett Birkhoff.
\newblock Tres observaciones sobre el algebra lineal.
\newblock \emph{Univ. Nac. Tucuman}, 1946.

\bibitem[Burges et~al.(2005)Burges, Shaked, Renshaw, Lazier, Deeds, Hamilton,
  and Hullender]{burges2005learning}
Chris Burges, Tal Shaked, Erin Renshaw, Ari Lazier, Matt Deeds, Nicole
  Hamilton, and Greg Hullender.
\newblock Learning to rank using gradient descent.
\newblock In \emph{ICML}, pages 89--96, 2005.

\bibitem[Calders et~al.(2009)Calders, Kamiran, and
  Pechenizkiy]{calders2009building}
Toon Calders, Faisal Kamiran, and Mykola Pechenizkiy.
\newblock Building classifiers with independency constraints.
\newblock In \emph{Data mining workshops, ICDMW}, pages 13--18, 2009.

\bibitem[Calmon et~al.(2017)Calmon, Wei, Vinzamuri, Ramamurthy, and
  Varshney]{calmon2017optimized}
Flavio~P Calmon, Dennis Wei, Bhanukiran Vinzamuri, Karthikeyan~Natesan
  Ramamurthy, and Kush~R Varshney.
\newblock Optimized pre-processing for discrimination prevention.
\newblock In \emph{NIPS}, 2017.

\bibitem[Carbonell and Goldstein(1998)]{Carbonell:1998:UMD:290941.291025}
Jaime Carbonell and Jade Goldstein.
\newblock The use of {MMR}, diversity-based reranking for reordering documents
  and producing summaries.
\newblock In \emph{SIGIR}, 1998.

\bibitem[Celis et~al.(2018)Celis, Straszak, and Vishnoi]{celis2017ranking}
L~Elisa Celis, Damian Straszak, and Nisheeth~K Vishnoi.
\newblock Ranking with fairness constraints.
\newblock \emph{ICALP}, 2018.

\bibitem[Celis et~al.(2020)Celis, Mehrotra, and
  Vishnoi]{celis2020interventions}
L~Elisa Celis, Anay Mehrotra, and Nisheeth~K Vishnoi.
\newblock Interventions for ranking in the presence of implicit bias.
\newblock In \emph{Proceedings of the Conference on Fairness, Accountability,
  and Transparency}, 2020.

\bibitem[Clarke et~al.(2008)Clarke, Kolla, Cormack, Vechtomova, Ashkan,
  B\"{u}ttcher, and MacKinnon]{Clarke:2008:NDI:1390334.1390446}
Charles~L.A. Clarke, Maheedhar Kolla, Gordon~V. Cormack, Olga Vechtomova, Azin
  Ashkan, Stefan B\"{u}ttcher, and Ian MacKinnon.
\newblock Novelty and diversity in information retrieval evaluation.
\newblock In \emph{SIGIR}, 2008.

\bibitem[Diaz et~al.(2020)Diaz, Mitra, Ekstrand, Biega, and
  Carterette]{diaz2020evaluating}
Fernando Diaz, Bhaskar Mitra, Michael~D Ekstrand, Asia~J Biega, and Ben
  Carterette.
\newblock Evaluating stochastic rankings with expected exposure.
\newblock In \emph{CIKM}, 2020.

\bibitem[Dvoretzky et~al.(1956)Dvoretzky, Kiefer, and
  Wolfowitz]{dvoretzky:kiefer:wolfowitz:CDF}
Aryeh Dvoretzky, Jack Kiefer, and Jacob Wolfowitz.
\newblock Asymptotic minimax character of the sample distribution function and
  of the classical multinomial estimator.
\newblock \emph{Annals of Mathematical Statistics}, 27\penalty0 (3):\penalty0
  642--669, 1956.

\bibitem[Dwork et~al.(2012)Dwork, Hardt, Pitassi, Reingold, and
  Zemel]{dwork2012fairness}
Cynthia Dwork, Moritz Hardt, Toniann Pitassi, Omer Reingold, and Richard Zemel.
\newblock Fairness through awareness.
\newblock In \emph{ITCS}. ACM, 2012.

\bibitem[Dwork et~al.(2019)Dwork, Kim, Reingold, Rothblum, and
  Yona]{dwork2019learning}
Cynthia Dwork, Michael~P Kim, Omer Reingold, Guy~N Rothblum, and Gal Yona.
\newblock Learning from outcomes: Evidence-based rankings.
\newblock In \emph{FOCS}, 2019.

\bibitem[Ghosh et~al.(2021)Ghosh, Dutt, and Wilson]{ghosh2021fair}
Avijit Ghosh, Ritam Dutt, and Christo Wilson.
\newblock When fair ranking meets uncertain inference.
\newblock \emph{arXiv preprint arXiv:2105.02091}, 2021.

\bibitem[Gini(1936)]{gini1936measure}
Corrado Gini.
\newblock On the measure of concentration with special reference to income and
  statistics.
\newblock \emph{Colorado College Publication, General Series}, 1936.

\bibitem[Hardt et~al.(2016)Hardt, Price, and Srebro]{hardt2016equality}
Moritz Hardt, Eric Price, and Nati Srebro.
\newblock Equality of opportunity in supervised learning.
\newblock In \emph{NIPS}, 2016.

\bibitem[Harper and Konstan(2015)]{harper2015movielens}
F~Maxwell Harper and Joseph~A Konstan.
\newblock The {MovieLens} datasets: History and context.
\newblock \emph{ACM TIIS}, 2015.

\bibitem[Heidari and Krause(2018)]{heidari2018preventing}
Hoda Heidari and Andreas Krause.
\newblock Preventing disparate treatment in sequential decision making.
\newblock In \emph{IJCAI}, 2018.

\bibitem[J{\"a}rvelin and Kek{\"a}l{\"a}inen(2002)]{jarvelin2002cumulated}
Kalervo J{\"a}rvelin and Jaana Kek{\"a}l{\"a}inen.
\newblock Cumulated gain-based evaluation of ir techniques.
\newblock \emph{TOIS}, 2002.

\bibitem[Joseph et~al.(2016)Joseph, Kearns, Morgenstern, and Roth]{Joseph2016}
Matthew Joseph, Michael Kearns, Jamie Morgenstern, and Aaron Roth.
\newblock Fairness in learning: Classic and contextual bandits.
\newblock In \emph{NIPS}, 2016.

\bibitem[Kallus and Zhou(2019)]{kallus2019fairness}
Nathan Kallus and Angela Zhou.
\newblock The fairness of risk scores beyond classification: Bipartite ranking
  and the {XAUC} metric.
\newblock \emph{NeurIPS}, 2019.

\bibitem[Kearns et~al.(2017)Kearns, Roth, and Wu]{Kearns2017}
Michael Kearns, Aaron Roth, and Zhiwei~Steven Wu.
\newblock {Meritocratic fairness for cross-population selection}.
\newblock In \emph{ICML}, 2017.

\bibitem[Kilbertus et~al.(2017)Kilbertus, Carulla, Parascandolo, Hardt,
  Janzing, and Sch{\"o}lkopf]{kilbertus2017avoiding}
Niki Kilbertus, Mateo~Rojas Carulla, Giambattista Parascandolo, Moritz Hardt,
  Dominik Janzing, and Bernhard Sch{\"o}lkopf.
\newblock Avoiding discrimination through causal reasoning.
\newblock In \emph{NIPS}, 2017.

\bibitem[Kim et~al.(2019)Kim, Ghorbani, and Zou]{kim2019multiaccuracy}
Michael~P Kim, Amirata Ghorbani, and James Zou.
\newblock Multiaccuracy: Black-box post-processing for fairness in
  classification.
\newblock In \emph{AAAI Conference on AI, Ethics, and Society}, 2019.

\bibitem[Kleinberg and Raghavan(2018)]{kleinberg2018selection}
Jon Kleinberg and Manish Raghavan.
\newblock Selection problems in the presence of implicit bias.
\newblock In \emph{ITCS}, 2018.

\bibitem[Kusner et~al.(2017)Kusner, Loftus, Russell, and
  Silva]{kusner2017counterfactual}
Matt~J Kusner, Joshua Loftus, Chris Russell, and Ricardo Silva.
\newblock Counterfactual fairness.
\newblock In \emph{NIPS}, 2017.

\bibitem[Lahoti et~al.(2019)Lahoti, Gummadi, and
  Weikum]{lahoti2019operationalizing}
Preethi Lahoti, Krishna~P Gummadi, and Gerhard Weikum.
\newblock Operationalizing individual fairness with pairwise fair
  representations.
\newblock \emph{VLDB Endowment}, 2019.

\bibitem[Liu et~al.(2017)Liu, Radanovic, Dimitrakakis, Mandal, and
  Parkes]{liu2017calibrated}
Yang Liu, Goran Radanovic, Christos Dimitrakakis, Debmalya Mandal, and David~C
  Parkes.
\newblock Calibrated fairness in bandits.
\newblock \emph{FATML}, 2017.

\bibitem[Lum and Johndrow(2016)]{lum2016statistical}
Kristian Lum and James Johndrow.
\newblock A statistical framework for fair predictive algorithms.
\newblock \emph{FATML}, 2016.

\bibitem[Massart(1990)]{massart:DKW}
Pascal Massart.
\newblock The tight constant in the {Dvoretzky}-{Kiefer}-{Wolfowitz}
  inequality.
\newblock \emph{Annals of Probability}, 18\penalty0 (3):\penalty0 1269--1283,
  1990.

\bibitem[Mehrabi et~al.(2019)Mehrabi, Morstatter, Saxena, Lerman, and
  Galstyan]{mehrabi2019survey}
Ninareh Mehrabi, Fred Morstatter, Nripsuta Saxena, Kristina Lerman, and Aram
  Galstyan.
\newblock A survey on bias and fairness in machine learning.
\newblock \emph{arXiv preprint arXiv:1908.09635}, 2019.

\bibitem[Mehrotra and Celis(2021)]{mehrotra2021mitigating}
Anay Mehrotra and L~Elisa Celis.
\newblock Mitigating bias in set selection with noisy protected attributes.
\newblock In \emph{Proceedings of the ACM Conference on Fairness,
  Accountability, and Transparency}, 2021.

\bibitem[Mehrotra et~al.(2018)Mehrotra, McInerney, Bouchard, Lalmas, and
  Diaz]{mehrotra2018towards}
Rishabh Mehrotra, James McInerney, Hugues Bouchard, Mounia Lalmas, and Fernando
  Diaz.
\newblock Towards a fair marketplace: Counterfactual evaluation of the
  trade-off between relevance, fairness \& satisfaction in recommendation
  systems.
\newblock In \emph{CIKM}, 2018.

\bibitem[Morik et~al.(2020)Morik, Singh, Hong, and
  Joachims]{morik2020controlling}
Marco Morik, Ashudeep Singh, Jessica Hong, and Thorsten Joachims.
\newblock Controlling fairness and bias in dynamic learning-to-rank.
\newblock In \emph{SIGIR}, 2020.

\bibitem[Narasimhan et~al.(2020)Narasimhan, Cotter, Gupta, and
  Wang]{narasimhan2020pairwise}
Harikrishna Narasimhan, Andrew Cotter, Maya Gupta, and Serena Wang.
\newblock Pairwise fairness for ranking and regression.
\newblock In \emph{AAAI}, 2020.

\bibitem[Patil et~al.(2020)Patil, Ghalme, Nair, and
  Narahari]{patil2020achieving}
Vishakha Patil, Ganesh Ghalme, Vineet Nair, and Y~Narahari.
\newblock Achieving fairness in the stochastic multi-armed bandit problem.
\newblock In \emph{AAAI}, 2020.

\bibitem[Penha and Hauff(2021)]{penha2021calibration}
Gustavo Penha and Claudia Hauff.
\newblock On the calibration and uncertainty of neural learning to rank models.
\newblock \emph{arXiv preprint arXiv:2101.04356}, 2021.

\bibitem[Pleiss et~al.(2017)Pleiss, Raghavan, Wu, Kleinberg, and
  Weinberger]{pleiss2017fairness}
Geoff Pleiss, Manish Raghavan, Felix Wu, Jon Kleinberg, and Kilian~Q
  Weinberger.
\newblock On fairness and calibration.
\newblock \emph{NIPS}, 2017.

\bibitem[Prost et~al.(2021)Prost, Awasthi, Blumm, Kumthekar, Potter, Wei, Wang,
  Chi, Chen, and Beutel]{prost2021measuring}
Flavien Prost, Pranjal Awasthi, Nick Blumm, Aditee Kumthekar, Trevor Potter,
  Li~Wei, Xuezhi Wang, Ed~H Chi, Jilin Chen, and Alex Beutel.
\newblock Measuring model fairness under noisy covariates: A theoretical
  perspective.
\newblock \emph{arXiv preprint arXiv:2105.09985}, 2021.

\bibitem[Radlinski et~al.(2009)Radlinski, Bennett, Carterette, and
  Joachims]{radlinski2009redundancy}
Filip Radlinski, Paul~N Bennett, Ben Carterette, and Thorsten Joachims.
\newblock Redundancy, diversity and interdependent document relevance.
\newblock In \emph{ACM SIGIR Forum}, 2009.

\bibitem[Robertson(1977)]{robertson1977probability}
Stephen~E Robertson.
\newblock The probability ranking principle in {IR}.
\newblock \emph{Journal of documentation}, 1977.

\bibitem[Schumann et~al.(2019)Schumann, Lang, Mattei, and
  Dickerson]{schumann2019group}
Candice Schumann, Zhi Lang, Nicholas Mattei, and John~P Dickerson.
\newblock Group fairness in bandit arm selection.
\newblock \emph{arXiv preprint arXiv:1912.03802}, 2019.

\bibitem[Singh and Joachims(2018)]{singh2018fairness}
Ashudeep Singh and Thorsten Joachims.
\newblock Fairness of exposure in rankings.
\newblock In \emph{KDD}, 2018.

\bibitem[Singh and Joachims(2019)]{singh2019policy}
Ashudeep Singh and Thorsten Joachims.
\newblock Policy learning for fairness in ranking.
\newblock In \emph{NeurIPS}, 2019.

\bibitem[Soufiani et~al.(2012)Soufiani, Parkes, and Xia]{soufiani2012random}
Hossein~Azari Soufiani, David~C Parkes, and Lirong Xia.
\newblock Random utility theory for social choice.
\newblock In \emph{NIPS}, 2012.

\bibitem[Taylor et~al.(2008)Taylor, Guiver, Robertson, and
  Minka]{taylor2008softrank}
Michael Taylor, John Guiver, Stephen Robertson, and Tom Minka.
\newblock Softrank: optimizing non-smooth rank metrics.
\newblock In \emph{WSDM}, 2008.

\bibitem[Woodworth et~al.(2017)Woodworth, Gunasekar, Ohannessian, and
  Srebro]{woodworth2017learning}
Blake Woodworth, Suriya Gunasekar, Mesrob~I Ohannessian, and Nathan Srebro.
\newblock Learning non-discriminatory predictors.
\newblock \emph{COLT}, 2017.

\bibitem[Yang and Stoyanovich(2017)]{yang2016measuring}
Ke~Yang and Julia Stoyanovich.
\newblock Measuring fairness in ranked outputs.
\newblock \emph{SSDBM}, 2017.

\bibitem[Zafar et~al.(2017)Zafar, Valera, Gomez~Rodriguez, and
  Gummadi]{zafar2017fairness}
Muhammad~Bilal Zafar, Isabel Valera, Manuel Gomez~Rodriguez, and Krishna~P
  Gummadi.
\newblock Fairness beyond disparate treatment \& disparate impact: Learning
  classification without disparate mistreatment.
\newblock In \emph{WWW}, 2017.

\bibitem[Zehlike and Castillo(2020)]{zehlike2020reducing}
Meike Zehlike and Carlos Castillo.
\newblock Reducing disparate exposure in ranking: A learning to rank approach.
\newblock In \emph{The Web Conference}, 2020.

\bibitem[Zehlike et~al.(2017)Zehlike, Bonchi, Castillo, Hajian, Megahed, and
  Baeza-Yates]{zehlike2017fa}
Meike Zehlike, Francesco Bonchi, Carlos Castillo, Sara Hajian, Mohamed Megahed,
  and Ricardo Baeza-Yates.
\newblock Fa*ir: A fair top-k ranking algorithm.
\newblock \emph{CIKM}, 2017.

\bibitem[Zemel et~al.(2013)Zemel, Wu, Swersky, Pitassi, and
  Dwork]{zemel2013learning}
Rich Zemel, Yu~Wu, Kevin Swersky, Toni Pitassi, and Cynthia Dwork.
\newblock Learning fair representations.
\newblock In \emph{ICML}, 2013.

\bibitem[Zliobaite(2015)]{zliobaite2015relation}
Indre Zliobaite.
\newblock On the relation between accuracy and fairness in binary
  classification.
\newblock \emph{FATML Workshop at ICML}, 2015.

\end{thebibliography}
